\pdfoutput=1

\documentclass[11pt]{article}

\usepackage[final]{acl}

\usepackage{times}
\usepackage{latexsym}

\usepackage[T1]{fontenc}

\usepackage[utf8]{inputenc}

\usepackage{microtype}

\usepackage{inconsolata}

\usepackage{graphicx}

\usepackage{subcaption}
\usepackage{enumitem}
\usepackage{xspace}
\usepackage{amsmath}
\usepackage{amssymb}

\usepackage{booktabs}
\usepackage{multirow}
\usepackage{threeparttable}

\usepackage{caption}
\usepackage{threeparttable}

\usepackage{colortbl}
\usepackage{xcolor}
\usepackage{tikz}

\usepackage{amsthm}
\usepackage{adjustbox}

\definecolor{lightcyan}{rgb}{0.88,1,1} 
\definecolor{lightgray2}{rgb}{0.9,0.9,0.9}
\definecolor{lightgreen}{rgb}{0.8, 0.95, 0.8}

\newcommand{\ours}{Temporal Experts Averaging}
\newcommand{\oursab}{TEA}
\newcommand{\eg}{e.g.,\xspace}

\newtheorem{lemma}{Lemma}
\newtheorem{proposition}{Proposition}

\newtheorem{assumption}{Assumption}

\title{Scaling Up Temporal Domain Generalization via \ours{}}

\author{
 \textbf{Aoming Liu\textsuperscript{1}},
 \textbf{Kevin Miller\textsuperscript{1}},
 \textbf{Venkatesh Saligrama\textsuperscript{1}},
 \textbf{Kate Saenko\textsuperscript{1}},
\\
 \textbf{Boqing Gong\textsuperscript{1}},
 \textbf{Ser-Nam Lim\textsuperscript{2}},
 \textbf{Bryan A. Plummer\textsuperscript{1}}
\\
\\
 \textsuperscript{1}Boston University,
 \textsuperscript{2}University of Central Florida
\\
 \small{
   \textbf{Correspondence:} \href{mailto:amliu@bu.edu}{amliu@bu.edu}
 }
}

\begin{document}
\maketitle

\begin{abstract}
Temporal Domain Generalization (TDG) aims to generalize across temporal distribution shifts, \eg lexical change over time.  Prior work often addresses this by predicting future model weights.  However, full model prediction is prohibitively expensive for even reasonably sized models. Thus, recent methods only predict the classifier layer, limiting generalization by failing to adjust other model components. To address this, we propose \ours{} (\oursab), a novel and scalable TDG framework that updates the entire model using weight averaging to maximize generalization potential while minimizing computational costs. Our theoretical analysis guides us to two steps that enhance generalization to future domains. First, we create expert models with functional diversity yet parameter similarity by fine-tuning a domain-agnostic base model on individual temporal domains while constraining weight changes. Second, we optimize the bias-variance tradeoff through adaptive averaging coefficients derived from modeling temporal weight trajectories in a principal component subspace. Expert's contributions are based on their projected proximity to future domains. Extensive experiments across 7 TDG benchmarks, 5 models, and 2 TDG settings shows \oursab{} outperforms prior TDG methods by up to 69\% while being up to 60x more efficient\footnote{Code: \url{https://github.com/zxcvfd13502/TEA}}.
\end{abstract}
\section{Introduction}

Temporal Domain Generalization (TDG)~\cite{Bai2022TemporalDG, nasery2021training, qin2022generalizing, xie2024evolving, xie2024enhancing, yong2023continuous, xie2024weight} aims to generalize to unseen future data under temporal distribution shift without retraining the models, as illustrated in Fig.~\ref{fig:tdg_example}. Unlike traditional Domain Generalization (DG), which lacks target domain information~\cite{Li2017DeeperBA, Muandet2013DomainGV, Li2018DomainGW, Li2018DeepDG}, TDG often leverages temporal patterns for prediction to better adapt the models for future domains by predicting model weights, such as forecasting NLP research trends~\cite{yao2022wild}. However, prior work has scaling limitations.
For example, early brute-force methods predict entire models, but have prohibitive computational costs when scaling up the model and dataset size~\cite{nasery2021training,Bai2022TemporalDG,qin2022generalizing}. As shown in Fig.~\ref{fig:classifier_only}, recent methods improve efficiency by only predicting the classifier~\cite{xie2024evolving,xie2024weight}, but sacrifice generalization capabilities from keeping other model components frozen. 
Thus, these methods struggle to surpass basic ERM baselines when scaling-up benchmarks~\cite{yao2022wild, Lin2022TheCB}.

\begin{figure}[t]
  \centering
    \begin{subfigure}{0.90\linewidth}
    \centering
    \includegraphics[width=0.99\linewidth]{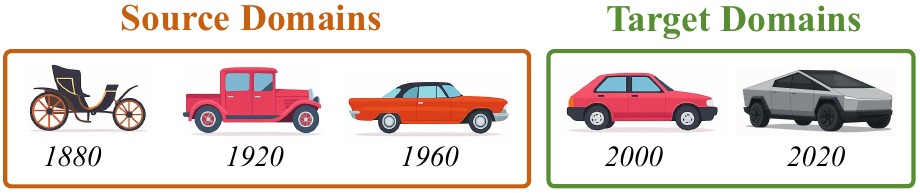}
    \caption{Temporal distribution shift in vehicle appearance.}
    \label{fig:vision_tdg}
  \end{subfigure}
  \\
  \vspace{1mm}
  \begin{subfigure}{0.90\linewidth}
    \centering
    \includegraphics[width=0.99\linewidth]{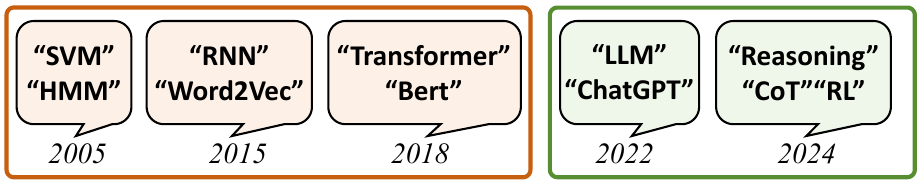}
    \caption{Temporal distribution shift in NLP paper keywords.}
    \label{fig:nlp_tdg}
  \end{subfigure}
    \vspace{-2mm}
    \caption{Examples of temporal domain generalization (TDG) span both \textbf{(a)} vision and \textbf{(b)} language tasks. TDG aims at enabling models trained on historical data to directly generalize to future data without retraining.}
    \label{fig:tdg_example}
\vspace{-4mm}
\end{figure}

\begin{figure*}[t]
  \centering
  \begin{subfigure}{0.28\linewidth}
    \centering
    \includegraphics[width=0.99\linewidth]{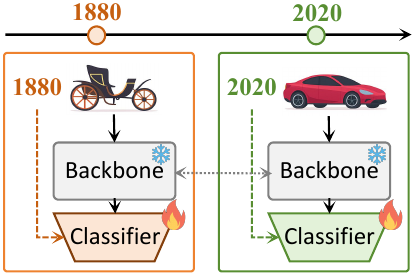}
    \caption{Classifier-only TDG framework}
    \label{fig:classifier_only}
  \end{subfigure}
  \hspace{0.01\linewidth}  
  \begin{tikzpicture}
    \draw[dashed, line width=0.5pt] (0,0) -- (0,3.5); 
  \end{tikzpicture}
  \hspace{0.01\linewidth}
  \begin{subfigure}{0.56\linewidth}
    \centering
    \includegraphics[width=0.99\linewidth]{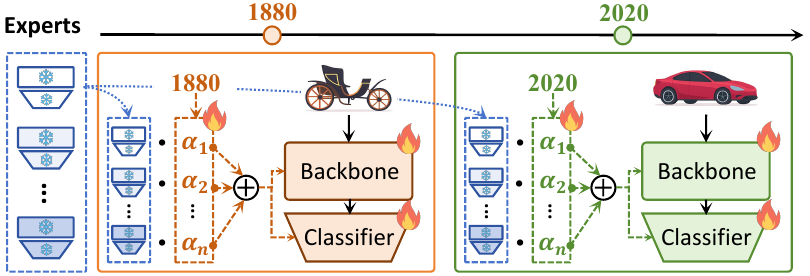}
    \caption{Our Temporal Expert Averaging (TEA) framework}
    \label{fig:our_tea}
  \end{subfigure}
  \vspace{-2mm}
    \caption{TDG framework comparison. \textbf{(a) Classifier-only TDG} \cite{xie2024evolving,xie2024weight} only predicts future classifiers to reduce computational costs in scaled-up scenarios, but limits generalization potential by neglecting other model components.  \textbf{(b) Our Temporal Expert Averaging (TEA)} enables higher generalization potential by adjusting the entire model through predicting future averaging coefficients of temporal experts capturing diverse functionalities. The low-dimensional nature of these coefficients ensures TEA's efficiency in scaled-up scenarios.}
  \label{fig:tdg_comp}
  \vspace{-4mm}
\end{figure*}
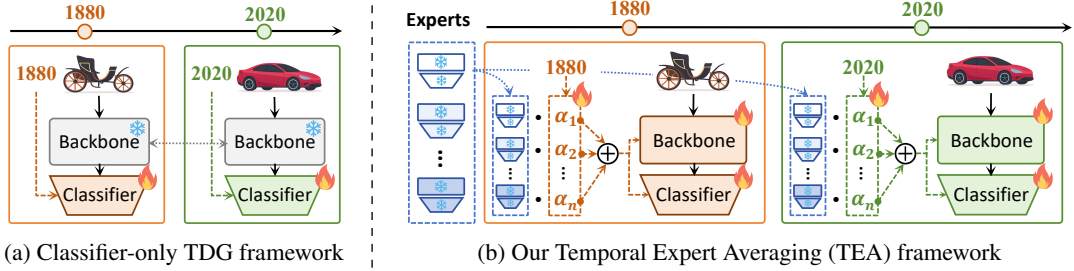
To address the scaling challenges, we propose \ours{} (\oursab), a TDG framework based on weight averaging (WA) that predicts the averaging coefficients of temporal experts for future domains. However, WA methods for DG, \eg~\cite{cha2021swad, rame2022diwa, modelsoup}, lack mechanisms to exploit temporal patterns in TDG. 
We identify two key capabilities required to leverage temporal patterns for WA through a bias-variance-covariance-locality decomposition analysis of generalization error.  First, model weights need functional diversity, yet parameter similarity. Second, optimizing averaging coefficients to achieve better bias-variance tradeoffs than uniform averaging.

To this end, our \oursab{} satisfies the first criteria by training a domain-agnostic base model on all source domains, followed by constrained incremental fine-tuning on each individual domain. To fulfill the second criteria, we  extract principal components from the deviations between expert weights and the base model, creating a low-dimensional subspace to model temporal weight trajectories. This enables forecasting future domain positions and averaging experts based on their projected proximity to the future domain.
This enables \oursab{} to temporally-adapt all model parameters with computational costs comparable to standard ERM training, offering higher generalization potential than merely adjusting the classifier.

The superiority of TEA is demonstrated through comprehensive evaluation across 7 diverse TDG benchmarks and 5 different models, covering both vision and language tasks. Beyond standard TDG with simultaneous access to all source domains, we also evaluate on Continual Domain Generalization over Temporal Drift (CDGTD) settings, where new domains arrive sequentially. Across this extensive evaluation, TEA consistently achieves new state-of-the-art results, outperforming prior TDG methods by up to 69\% while being up to 60x more efficient.

Our contributions can be summarized as follows:

\begin{itemize}[nosep,leftmargin=*]
    \item We propose \oursab, a novel weight-averaging-based TDG framework that efficiently enhances generalization across temporal shifts with broad model/dataset compatibility.
    \item We provide valuable theoretical insights on the under-explored WA-TDG integration, design our method based on these insights, and validate our insights and method through superior generalization performance across various benchmarks.
    \item We enhance TDG evaluation comprehensiveness by addressing both TDG and CDGTD, unlike prior work that typically focused on just one setting. This includes introducing CLEAR-10 and CLEAR-100 \citep{Lin2022TheCB} as new evaluation benchmarks for TDG.
\end{itemize}

\section{Related Work}

\noindent\textbf{Temporal Domain Generalization (TDG)} \cite{OrtizJimnez2019CDOTCD, mancini2019adagraph, Wang2020ContinuouslyID, Bai2022TemporalDG, nasery2021training, Zeng2023ForeseeWY, wang2022evolving, xie2024evolving, xie2024enhancing, yong2023continuous, xie2024weight} exploits temporal patterns in ordered domains with smooth distribution shifts to enhance generalization to future domains. Early approaches like GI~\cite{nasery2021training} and DRAIN~\cite{Bai2022TemporalDG} predict entire model parameters, but face computational challenges with large-scale models, while recent methods like EvoS~\cite{xie2024evolving} and W-Diff~\cite{xie2024weight} reduce costs by only adjusting classifiers, potentially limiting generalization. TDG encompasses multiple settings: the original setting with simultaneous access to all source domains, Continual Domain Generalization over Temporal Drift (CDGTD) with sequentially available domains, and Continuous Temporal Domain Generalization (CTDG) for continuously distributed temporal data. We focus on the original TDG and CDGTD settings as CTDG remains impractical for most realistic benchmarks.

\smallskip
\noindent\textbf{Domain Adaptation and Generalization.}  Enabling models to perform well on out-of-distribution (OOD) data has been a crucial challenge in machine learning. Two specific tasks highly relevant to our work are Domain Adaptation (DA) and Domain Generalization (DG). DA methods~\cite{saenko2010adapting, Sun2015ReturnOF, Sun2016DeepCC, Gong2012GeodesicFK, Tzeng2017AdversarialDD, li2016revisiting} typically adapt models against distribution shift by utilizing data from the target domain. In contrast, DG methods~\cite{Li2017DeeperBA, Muandet2013DomainGV, Li2018DomainGW, Li2017LearningTG, gulrajanisearch, Li2018DeepDG, Li2019EpisodicTF} operate without target domain information, solely leveraging source domain patterns to enhance OOD generalization. 

\smallskip
\noindent\textbf{Weight Averaging} (WA)~\cite{cha2021swad, cha2022miro, rame2022diwa, modelsoup} proves effective for Domain Generalization, with DiWA~\cite{rame2022diwa} showing reduced variance against marginal distribution shifts. While WA is also used in Multi-task Learning~\cite{ilharco2022patching,yadav2023ties,ortiz2023task,wang2024sam,stoica2023zipit}, with our design partly inspired by task arithmetic~\cite{ilharco2022editing}, fundamental differences between MTL and TDG make direct application impractical.
\section{Temporal Experts Averaging}
\label{sec:method}

\begin{figure*}[t]
  \centering
  \includegraphics[width=0.93\linewidth]{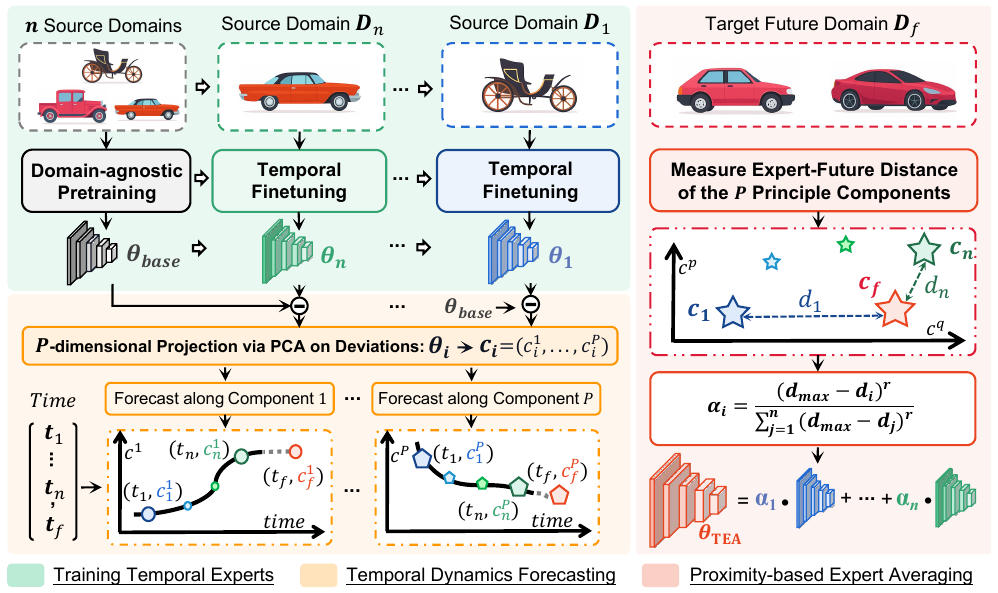}
  \vspace{-2mm}
    \caption{Overview of our TEA framework. Firstly, we obtain a base model $\theta_{\text{base}}$ through domain-agnostic pretraining on all source domains, then derive experts $\theta_1,...,\theta_n$ via constrained domain-specific incremental finetuning in reverse temporal order. Secondly, we apply PCA to expert weight deviations $\{\theta_i-\theta_{\text{base}}\}_{i=1}^{n}$, forecast future positions along the $P$ most significant components with Autoregressive Integrated Moving Average (ARIMA), effectively projecting experts into a low-dimensional space for prediction. Finally, we assign averaging coefficients based on projected expert-future proximity, where closer experts receive higher coefficients.}
  \label{fig:tdg_overview}
  \vspace{-4mm}
\end{figure*}

Let $\mathcal{X}$ be the input space, $\mathcal{Y}$ the label space, $\ell: \mathcal{Y}^2 \rightarrow \mathbb{R}^+$ a loss function, $\{D_i\}$ a sequence of domains with timestamps $t_i \in \mathcal{T}$ and distributions $p_i$.
Given source domains $\mathbf{D}_S =\{D_{i}\}_{i=1}^{S}$, where $t_1<\ldots<t_S$, 
and a neural network $f(\cdot, \theta): \mathcal{X} \rightarrow \mathcal{Y}$ with weights $\theta$, we aim to minimize the generalization error at future time $t_f > t_S$:
\begin{equation}
\mathcal{E}_{f}(\theta) = \mathbb{E}_{(x,y)\sim p_{f}}[\ell(f(x, \theta), y)].
\label{eq:future_error}
\end{equation}
 We obtain the weights of $S$ temporal expert models $\{\theta_i\}_{i=1}^S = \{\theta(l_{i})\}_{i=1}^S$, where $\theta_i$ is optimized for domain $D_{i}$ while using data from other domains, with learning procedure noted as $l_{i}=\{\{D_{i}\}_{i=1}^{S},t_i, c\}$ and other configurations (\eg hyper-parameters) as $c$. We leverage temporal patterns to derive adaptive coefficients $\{\alpha_i\}_{i=1}^S$, where $\sum_{i=1}^S \alpha_i = 1$ and $\alpha_i \geq 0$, for combining expert weights into the final weight $\theta_{\text{TEA}}$, formulated as:
\begin{align}
    f_{\text{TEA}} &\triangleq f(\cdot, \theta_{\text{TEA}}), \nonumber\\
    \theta_{\text{TEA}} &\triangleq \sum_{i=1}^S \alpha_i\left(\{t_i\}_{i=1}^S, \{\theta_i\}_{i=1}^S,t_f\right) \cdot \theta_i.
\end{align}
\smallskip
To leverage temporal shift patterns for reducing future generalization error, we gain insight into \oursab{} through theoretical analysis in Section~\ref{sec:insights}. Following the insights, we implement our \oursab{} by creating functionally diverse yet parametrically similar experts $\{\theta_i\}_{i=1}^S$ (Section~\ref{sec:train_expert}) and determining coefficients $\{\alpha_i\}_{i=1}^S$ based on expert-future proximity (Section~\ref{sec:pca_estimate_avg}). Section~\ref{sec:TWA_CDTDG} describes how we adapt \oursab{} to the CDGTD setting.

\subsection{Theoretical Analysis and Insights}
\label{sec:insights}

To gain insight into \oursab{}, we extend DiWA's~\cite{rame2022diwa} theoretical analysis developed for DG to our WA-TDG integration setting. Since our primary goal is to guide method design, we briefly summarize the theoretical analysis and results in the main text, with complete derivations and proofs available in Appendix~\ref{sec:theory}.

\noindent\textbf{Bias-variance-covariance-locality Decomposition.} Similar to DiWA~\cite{rame2022diwa}, we introduce the bias-variance-covariance-locality (BVCL) decomposition of generalization error for TDG and TEA by leveraging the similarity between averaging in weight space and function space. Denoting $\mathbb{E}_{f} = \mathbb{E}_{(x,y)\sim p_{f}}$, $\mathbf{l} = \{l_1,\ldots,l_S\}$, $\bar{f}_{i}(x) = \mathbb{E}_{l_{i}}[f(x, \theta(l_{i}))]$, $\text{bias}_i = y - \bar{f}_{i}(x)$, $\text{var}_i=\mathbb{E}_{l_{i}}\left[ \left(f(x, \theta(l_{i})) - \bar{f}_{i}(x)\right)^2 \right], \text{cov}_{i,j} = \mathbb{E}_{l_{i},l_{j}}\left[ \left(f(x, \theta(l_{i})) - \bar{f}_{i}(x)\right) \left(f(x, \theta(l_{j})) - \bar{f}_{j}(x)\right) \right]$ and $\Delta_{\{\theta\}} = \max_{i=1}^S \|\theta_i - \theta_{\text{TEA}}\|_2$, the expected generalization error on future timestamp $t_f$ of $\theta_{\text{TEA}} = \sum_{i=1}^S \alpha_i \theta_i$ over the joint distribution of the learning procedures is:
\begin{align}
&\mathbb{E}_{\mathbf{l}}[\mathcal{E}_{f}(\theta_{\text{TEA}})] = \mathbb{E}_{f}\left[ \mathcal{B} + \mathcal{V} +\mathcal{C}  \right] + O(\bar{\Delta}^2), \label{eq:bvcl}\\
&\mathcal{B}=\left(\sum_{i=1}^S \alpha_i \cdot \text{bias}_i\right)^2,\ \mathcal{V}=\sum_{i=1}^S \alpha_i^2 \cdot \text{var}_i, \nonumber\\
&\mathcal{C} = \sum_{i \neq j} \alpha_i \alpha_j \text{cov}_{i,j},\ 
\bar{\Delta}^2 = \mathbb{E}_{\mathbf{l}}\left[\Delta_{\{\theta\}}\right] .\nonumber
\end{align}

To reduce future generalization error in Equation~\ref{eq:bvcl}, we can control learning procedures $\{l_{i}\}_{i=1}^S$ affecting expert weights $\{\theta_i\}_{i=1}^S$ and modify averaging coefficients $\{\alpha_i\}_{i=1}^S$, which constitute the key differences between our TEA and WA for typical DG. While finding optimal solutions remains challenging due to real-world complexity, qualitative analysis provides valuable insights summarized as two tradeoffs implemented through experts and coefficients respectively. See Appendix~\ref{sec:theory} for detailed analysis and assumptions.

\smallskip
\noindent\textbf{Insight 1} \textit{Tradeoff between Functional Diversity and Parameter Similarity among Experts.} Covariance $\mathcal{C}$ reduction necessitates functional diversity among experts, while the locality constraint $\bar{\Delta}^2$ demands parameter similarity among experts.

\smallskip
\noindent\textbf{Insight 2} \textit{Tradeoff between Bias and Variance via Averaging Coefficients.} Reducing variance $\mathcal{V}$ requires averaging weights evenly, while reducing bias $\mathcal{B}$ demands concentrating coefficients on experts with lower bias magnitudes on future data.

\subsection{Training Temporal Experts}
\label{sec:train_expert}

TDG assumes smooth temporal distribution shifts with moderate changes between adjacent domains. This allows an expert to be fine-tuned for learning domain-specific functionality of neighboring domains with minimal parameter adjustments. Therefore, we can satisfy Insight 1 through incremental domain-specific fine-tuning while constraining minimal parameter changes. However, a prerequisite is that experts must have already thoroughly learned the intrinsic distribution.

A "pretraining-finetuning" approach is adopted for our expert training that efficiently generates diverse temporal experts with similar parameters. The overall process can be formulated as:
\begin{align}
\theta_{\text{base}}&=\theta_{S+1} = \theta(l_{\text{ERM}}(\textbf{D}_{S})),  \tag{Pretraining}\\
\theta_{i} &= \theta(l_{\text{SI}}(\{D_{t_i}\}, \theta_{i+1})), \tag{Finetuning}
\end{align}
where  $i \in \{S, \ldots, 1\}, \textbf{D}_{S}=\{D_{1}, \ldots, D_{S}\}$, $l_{\text{ERM}}$ represents the Empirical Risk Minimization (ERM) learning process, and $l_{\text{SI}}$ represents the learning process with Synaptic Intelligence (SI)~\cite{Zenke2017ContinualLT} constraining parameter changes.
\smallskip

\smallskip
\noindent\textbf{Pretraining} aims to capture intrinsic, time-invariant distributions. We apply standard ERM training with source domains $\{D_{1}, \ldots, D_{S}\}$. No temporal information is incorporated during this stage. Unlike WA for DG~\cite{cha2021swad, cha2022miro, rame2022diwa, modelsoup}, we update normalization layers during pretraining to prevent underfitting, as TDG exhibits smaller distribution differences than DG settings.

\smallskip
\noindent\textbf{Temporal Finetuning} sequentially adapts the base model to capture time-varying distributions. We freeze the normalization layers and proceed in reverse temporal order ($t_S \rightarrow \ldots \rightarrow t_1$) in this stage. For each domain $D_i$, we uniformly sample $K$ weights during finetuning, $\{\theta_{i}^{k}\}_{k=1}^{K}$, and expert $\theta_i$ is obtained by uniform averaging: 
$\theta_{i} = \sum_{k=1}^{K} \frac{1}{K} \theta_{i}^{k}$

SI~\cite{Zenke2017ContinualLT} is used to constrain parameter changes, which also prevent catastrophic forgetting of intrinsic distributions, but other continual learning methods can also be used. Since later fine-tuning stages are influenced by previous ones, we use reverse temporal order (recent to earliest) to better capture distributions from recent domains that more likely resemble future test domains under smooth distribution shift assumptions.

\subsection{Adaptive Weight Averaging}
\label{sec:pca_estimate_avg}

If future weights are available, we could satisfy Insight 2 by assigning coefficients based on expert-future weight proximity. 
However, precisely predicting the future in high-dimensional weight space is both hard and computationally prohibitive. In contrast, our \oursab{} approach only needs relative rankings of expert-future proximity.  Thus, we project experts into a low-dimensional space that captures the principal components of weight temporal evolution, enabling us to predict future positions and measure expert-future proximity efficiently for assigning averaging coefficients.


\smallskip
\noindent\textbf{PCA over Temporal Weight Deviation.} The weight deviations $\{\delta\theta_i\}_{i=1}^S,\ \delta\theta_i = \theta_i - \theta_{\text{base}}$ of all experts estimates weight dynamics under temporal distribution shifts. 
We apply PCA to $\{\delta\theta_i\}_{i=1}^{S}$ to decompose the principal components of weight temporal evolution and reduce noise. By considering only the $P$ most significant components $\{v_p\}_{p=1}^P$, we can obtain a $P$-dimensional space and project the experts into points in this space, where $\mathbf{c}_i=(c_i^1,...,c_i^P)$ is the projection of $\theta_i$:
\begin{align}
\mathbf{c}_i &= (c_i^1,...,c_i^P) \\
&= \left( \langle \theta_i - \theta_{\text{base}}, v_1 \rangle, \ldots, \langle \theta_i - \theta_{\text{base}}, v_P \rangle \right) \nonumber
\end{align}

\noindent\textbf{Principal Component Trajectory Forecasting.} We construct a temporal evolution trajectory of the $P$ principle components using all experts' projected points and their timestamps, $\{(\mathbf{c}_i,t_i)\}_{i=1}^{S}$. Then we predict the future domain position in this $P$-dimensional space by forecasting along this temporal evolution trajectory. As we often have limited temporal domains available leading to few historical points in the trajectory, we simply model the $P$-dimensional trajectory as $P$ separate time series, $\{(c_i^p,t_i)\}_{i=1}^{P} \text{ for } p \in \{1,...,P\}$, by treating all the dimensions independently. For prediction, we apply the Autoregressive Integrated Moving Average (ARIMA) model to each time series:
\begin{align}
c^p(t_f) = \text{ARIMA}(\{(c_i^p,t_i)\}_{i=1}^{S}, t_f) 
\end{align}
where $p \in \{1,...,P\}$ and $t_f$ is the future domain's timestamp. The predicted future point in the principle component space is $\mathbf{c}_f = (c^1(t_f),..., c^P(t_f))$.

\smallskip
\noindent\textbf{Distance-based Averaging Coefficients.} Based on Insight 2, we assign higher averaging coefficients to experts with greater expert-future proximity (lower expert-future distance) in the principal component space. Specifically, for expert $\theta_i$ with projected point $\mathbf{c}_i = (c_i^1,...,c_i^P)$ and our predicted future point $\mathbf{c}_f = (c^1(t_f),..., c^P(t_f))$, we calculate distance $d_i=\|\mathbf{c}_i - \mathbf{c}_f\|$. We then assign the averaging coefficient for $\theta_i$ as:
\begin{align}
\alpha_i = \frac{(d_{\max} - d_i)^r}{\sum_{j=1}^n (d_{\max} - d_j)^r},
\end{align}
where $d_{\max} = \max(d_1,...,d_n)$ and $r$ is a hyperparameter controlling the sharpness of the weighting distribution. Higher $r$ concentrates the averaging more on experts closer to the predicted future.

\subsection{TEA for CDTDG}
\label{sec:TWA_CDTDG}

The TEA method targets the original TDG setting with simultaneous access to all source domains, while the CDTDG setting only provides sequential access to domains. Simply sampling models during incremental learning fails because adjacent domains exhibit both temporal distribution shifts and newly introduced data variations, causing significant parameter differences that violate the locality constraints. 

We, therefore, slightly relax the CDTDG constraints by maintaining small memory buffers (\eg 10\%) of seen training data $\{d_{1}, d_{2}, \ldots, d_{S}\}$ from each domain. After training sequentially on all source domains, we can utilize the data in these stored buffers for our temporal finetuning and forecasting, avoiding the expensive cost of storing and revisiting full training data. Based on this relaxation, we apply the original TEA framework:
\begin{align}
&\theta_{\text{base}} = \theta_{S+1}= \theta(l_{\text{IncERM}}(\{\mathbf{D}_{S}\})) \nonumber \tag{Pretraining}\\
&\theta_{i} = \theta(l_{\text{SI}}(\{d_{t_i}\}, \theta_{i+1})),\nonumber \tag{Finetuning}
\end{align}
where $i \in \{S, \ldots, 1\}$, $l_{\text{IncERM}}$ is the incremental learning process with ERM, and $l_{\text{SI}}$ is the learning process with SI constraining parameter changes.


\begin{table*}[t]
\centering
\resizebox{\textwidth}{!}{
\begin{tabular}{c|c|ccccccccc|c}
\hline\hline
\multirow{2}{*}{Dataset} & \multirow{2}{*}{Metric} & \multicolumn{9}{c|}{Method} & \multirow{2}{*}{\textbf{TEA (ours)}} \\
\cline{3-11}
 & & ERM & IRM & CORAL & Mixup & LISA & GI$^{\S}$ & LSSAE$^{\S}$ & SWAD & DiWA & \\
\hline\hline
\multirow{3}{*}{\shortstack{Yearbook\\\citep{yao2022wild}}} 
 & $D_{S + 1}$ & 89.30 & \underline{97.09} & 95.94 & 94.98 & 95.51 & 97.42 & 93.93 & 97.18 & 97.66 &  \cellcolor{lightgreen!50}\textbf{97.71} \\
 & $\text{OOD}_{\text{avg.}}$ & 88.46 & 94.52 & 91.79 & 91.12 & 92.97 & \textbf{96.37} & 92.12 & 95.00 & \underline{95.36} & \cellcolor{lightgreen!50}95.95 \\
 & $\text{OOD}_{\text{worst}}$ & 86.81 & 92.58 & 88.84 & 88.35 & 91.29 & \textbf{95.73} & 88.75  & \underline{93.89} & 94.42&\cellcolor{lightgreen!50}94.80 \\
\hline
\multirow{3}{*}{RMNIST} 
 & $D_{S + 1}$ & \underline{98.15} & 95.10 & 93.04 & 97.11 & 96.21 & 97.78 & 96.73 & 97.93 & 97.67 &\cellcolor{lightgreen!50}\textbf{98.61} \\
 & $\text{OOD}_{\text{avg.}}$ & \underline{92.14} & 85.05 & 79.10 & 89.66 & 87.04 & 91.00 & 90.36 & 94.51 & 92.06 &\cellcolor{lightgreen!50}\textbf{94.47} \\
 & $\text{OOD}_{\text{worst}}$ & \underline{83.89} & 72.52 & 62.96 & 79.63 & 75.15 & 82.46 & 82.13 & 84.89 & 84.31 &\cellcolor{lightgreen!50}\textbf{88.83} \\
\hline
\multirow{3}{*}{\shortstack{FMoW\\\citep{yao2022wild}}} 
 & $D_{S + 1}$ & \underline{72.43} & 64.77 & 62.14 & 70.27 & 70.05 & 61.62 & 59.15 & 71.59 & 73.85 &  \cellcolor{lightgreen!50}\textbf{75.63} \\
 & $\text{OOD}_{\text{avg.}}$ & \underline{59.76} & 54.92 & 51.42 & 57.73 & 55.52 & 50.83 & 48.66 & 59.96 & 60.77 &  \cellcolor{lightgreen!50}\textbf{62.45} \\
 & $\text{OOD}_{\text{worst}}$ & \underline{49.85} & 46.51 & 42.19 & 48.04 & 44.61 & 42.78 & 41.38 & 50.48 & 51.00 & \cellcolor{lightgreen!50}\textbf{52.45} \\
\hline
\multirow{3}{*}{\shortstack{CLEAR-10\\\citep{Lin2022TheCB}}} 
 & $D_{S + 1}$ & \underline{80.83} & 77.50 & 77.57 & 78.57 & 71.50 & 72.73 & 55.63 & 69.20 & 81.03 &\cellcolor{lightgreen!50}\textbf{83.53} \\
 & $\text{OOD}_{\text{avg.}}$ & \underline{81.20} & 77.03 & 77.89 & 78.21 & 70.89 & 71.31 & 55.74 & 68.14 & 81.17 &\cellcolor{lightgreen!50}\textbf{83.16} \\
 & $\text{OOD}_{\text{worst}}$ & \underline{80.83} & 76.60 & 77.47 & 76.90 & 70.27 & 70.33 & 54.83 & 67.53 & 80.60 &\cellcolor{lightgreen!50}\textbf{82.43} \\
\hline
\multirow{3}{*}{\shortstack{CLEAR-100\\\citep{Lin2022TheCB}}} 
 & $D_{S + 1}$ & \underline{63.92} & 57.74 & 61.95 & 62.96 & 53.80 & 51.87 & 39.82 & 47.38 & 65.64 &\cellcolor{lightgreen!50}\textbf{67.39} \\
 & $\text{OOD}_{\text{avg.}}$ & \underline{63.19} & 56.79 & 60.53 & 62.42 & 52.82 & 51.06 & 39.41 & 46.04 & 64.71 &\cellcolor{lightgreen!50}\textbf{66.96} \\
 & $\text{OOD}_{\text{worst}}$ & \underline{62.62} & 56.24 & 59.46 & 61.93 & 52.08 & 50.32 & 38.87 & 45.18 & 63.96 &\cellcolor{lightgreen!50}\textbf{66.43} \\
\hline
\multirow{3}{*}{\shortstack{Huffpost\\\citep{yao2022wild}}} 
 & $D_{S + 1}$ & 72.74 & 71.04 & 71.34 & \underline{73.34} & 72.19 & 68.06 & - & 73.40 & 73.31 & \cellcolor{lightgreen!50}\textbf{73.43} \\
 & $\text{OOD}_{\text{avg.}}$ & \underline{71.50} & 70.31 & 70.08 & 71.16 & 70.24 & 66.32 & - & 71.59 & 71.51 &\cellcolor{lightgreen!50}\textbf{72.12} \\
 & $\text{OOD}_{\text{worst}}$ & \underline{69.63} & 68.97 & 68.68 & 69.29 & 68.60 & 64.64 & - & 70.10 & 70.18 &\cellcolor{lightgreen!50}\textbf{70.64} \\
\hline
\multirow{3}{*}{\shortstack{Arxiv\\\citep{yao2022wild}}} 
 & $D_{S + 1}$ & 57.49 & 51.11 & 50.98 & \underline{57.58} & 56.53 & 53.43 & - & 57.08 & 57.21 & \cellcolor{lightgreen!50}\textbf{59.28} \\
 & $\text{OOD}_{\text{avg.}}$ & 52.38 & 45.89 & 45.77 & \underline{52.77} & 52.41 & 49.19 & - & 52.96 & 52.80 &\cellcolor{lightgreen!50}\textbf{55.23} \\
 & $\text{OOD}_{\text{worst}}$ & 49.28 & 42.86 & 42.71 & 49.62 & \underline{49.67} & 46.13 & - & 50.09 & 49.92 &\cellcolor{lightgreen!50}\textbf{52.31} \\
\hline
\multirow{3}{*}{\textbf{Overall Avg.}} 
 & $D_{S + 1}$ & \cellcolor{lightcyan!50}76.41 & \cellcolor{lightcyan!50}73.48 & \cellcolor{lightcyan!50}73.28 & \cellcolor{lightcyan!50}\underline{76.40} & \cellcolor{lightcyan!50}73.68 & \cellcolor{lightcyan!50}71.70 & \cellcolor{lightcyan!50}- & \cellcolor{lightcyan!50}73.25 & \cellcolor{lightcyan!50}\underline{77.91} &\cellcolor{lightgreen!50}\textbf{79.37} \\
& $\text{OOD}_{\text{avg.}}$ & \cellcolor{lightcyan!50}72.66 & \cellcolor{lightcyan!50}69.22 & \cellcolor{lightcyan!50}68.08 & \cellcolor{lightcyan!50}\underline{71.87} & \cellcolor{lightcyan!50}69.13 & \cellcolor{lightcyan!50}67.87 & \cellcolor{lightcyan!50}- & \cellcolor{lightcyan!50}69.74 & \cellcolor{lightcyan!50}\underline{74.77}& \cellcolor{lightgreen!50}\textbf{75.76} \\
 & $\text{OOD}_{\text{worst}}$ & \cellcolor{lightcyan!50}\underline{69.13} & \cellcolor{lightcyan!50}64.90 & \cellcolor{lightcyan!50}62.04 & \cellcolor{lightcyan!50}67.54 & \cellcolor{lightcyan!50}65.95 & \cellcolor{lightcyan!50}64.63 & \cellcolor{lightcyan!50}- & \cellcolor{lightcyan!50}65.71 & \cellcolor{lightcyan!50}\underline{70.55} &  \cellcolor{lightgreen!50}\textbf{72.56} \\

\hline\hline
\end{tabular}
}
\vspace{-2mm}
\caption{Accuracy (\%) on all benchmarks under TDG setting. Baselines include ERM, IRM~\cite{arjovsky2019invariant}, CORAL~\cite{Sun2016DeepCC}, Mixup~\cite{zhang2018mixup}, LISA~\cite{yao2022improving}, GI~\cite{nasery2021training}, LSSAE~\cite{qin2022generalizing}, SWAD~\cite{cha2021swad}, and DiWA~\cite{rame2022diwa}. Best and second-best results are \textbf{bolded} and \underline{underlined}. For FMoW, CLEAR-10\&100, Huffpost and Arxiv, we only apply GI to classifiers due to backbone size limitations. LSSAE only applies to image benchmarks. ${\S}$ indicates TDG baselines.}
\label{tab:tdg_results}
\vspace{-4mm}
\end{table*}
\begin{table*}[t]
\centering

\resizebox{\textwidth}{!}{
\begin{tabular}{c|c|cccccccccc|c}
\hline\hline
\multirow{2}{*}{Dataset} & \multirow{2}{*}{Metric} & \multicolumn{10}{c|}{Method} & \multirow{2}{*}{\textbf{TEA (ours)}} \\
\cline{3-12}
 & & IncERM & Mixup & SimCLR & SwAV & EWC & SI & A-GEM & DRAIN$^{\S}$ & EvoS$^{\S}$ & W-Diff$^{\S}$ & \\
\hline\hline
\multirow{3}{*}{\shortstack{Yearbook\\\citep{yao2022wild}}} 
 & $D_{S + 1}$ & 96.61 & 90.21 & 95.94 & 97.37 & 97.18 & 97.09 & 94.36 & 96.23 & 97.37 & 97.32 & \cellcolor{lightgreen!50}\textbf{97.75} \\
 & $\text{OOD}_{\text{avg.}}$ & 94.72 & 89.83 & 93.07 & 94.27 & 95.12 & 94.67 & 90.96 & 94.71 & \textbf{95.53} & 95.03 & \cellcolor{lightgreen!50}\underline{95.29} \\
 & $\text{OOD}_{\text{worst}}$ & 93.48 & 88.43 & 89.65 & 91.44 & 93.64 & 93.48 & 88.88 & 93.73 & \textbf{94.78} & 94.05 & \cellcolor{lightgreen!50}\underline{94.40} \\
\hline
\multirow{3}{*}{RMNIST} 
 & $D_{S + 1}$ & 98.62 & 98.43 & 98.23 & 98.08 & 98.56 & 98.61 & 95.99 & 98.52 & 98.64 & \underline{98.70} & \cellcolor{lightgreen!50}\textbf{98.74} \\
 & $\text{OOD}_{\text{avg.}}$ & 92.80 & 92.38 & 90.98 & 90.85 & 92.02 & 93.27 & 86.95 & 93.09 & \underline{93.84} & \textbf{94.12} & \cellcolor{lightgreen!50}93.76 \\
 & $\text{OOD}_{\text{worst}}$ & 84.61 & 83.45 & 81.05 & 80.96 & 82.80 & 85.65 & 75.45 & 85.75 & 87.04 & \textbf{87.36} & \cellcolor{lightgreen!50}\underline{87.05} \\
\hline
\multirow{3}{*}{\shortstack{FMoW\\\citep{yao2022wild}}} 
 & $D_{S + 1}$ & 65.52 & 64.84 & 64.97 & 66.47 & 66.23 & 66.61 & 54.54 & 67.22 & 67.18 & \textbf{68.80} & \cellcolor{lightgreen!50}\underline{67.87} \\
 & $\text{OOD}_{\text{avg.}}$ & 53.99 & 52.00 & 53.20 & 54.51 & 54.55 & 54.89 & 47.61 & 55.05 & 54.64 & \textbf{55.86} & \cellcolor{lightgreen!50}\underline{55.21} \\
 & $\text{OOD}_{\text{worst}}$ & 45.23 & 42.54 & 44.71 & 45.29 & 45.80 & \underline{46.46} & 41.13 & 46.24 & 45.86 & \textbf{46.51} & \cellcolor{lightgreen!50}46.27 \\
\hline
\multirow{3}{*}{\shortstack{CLEAR-10\\\citep{Lin2022TheCB}}} 
 & $D_{S + 1}$ & 75.90 & 74.97 & 78.43 & 77.53 & 75.07 & 76.73 & 60.67 & 74.40 & 77.03 & 68.00 & \cellcolor{lightgreen!50}\textbf{79.20} \\
 & $\text{OOD}_{\text{avg.}}$ & 75.82 & 74.99 & \textbf{78.41} & \underline{78.05} & 73.71 & 76.07 & 59.49 & 74.52 & 77.06 & 67.85 & \cellcolor{lightgreen!50}77.87 \\
 & $\text{OOD}_{\text{worst}}$ & 74.83 & 74.10 & \textbf{77.73} & 77.13 & 72.30 & 75.00 & 58.17 & 73.97 & 76.87 & 66.03 & \cellcolor{lightgreen!50}\underline{77.43} \\
\hline
\multirow{3}{*}{\shortstack{CLEAR-100\\\citep{Lin2022TheCB}}} 
 & $D_{S + 1}$ & 56.73 & 51.68 & \textbf{60.52} & 58.89 & 56.22 & 31.76 & 23.61 & 54.74 & 57.02 & 52.33 & \cellcolor{lightgreen!50}\underline{58.93} \\
 & $\text{OOD}_{\text{avg.}}$ & 55.67 & 50.86 & \textbf{59.67} & 57.59 & 55.20 & 30.82 & 22.55 & 53.16 & 56.09 & 51.92 & \cellcolor{lightgreen!50}\underline{58.43} \\
 & $\text{OOD}_{\text{worst}}$ & 54.47 & 50.32 & \textbf{58.65} & 56.53 & 54.30 & 30.35 & 21.64 & 51.90 & 55.47 & 51.65 & \cellcolor{lightgreen!50}\underline{57.70} \\
\hline
\multirow{3}{*}{\shortstack{Huffpost\\\citep{yao2022wild}}} 
 & $D_{S + 1}$ & 73.57 & 73.07 & - & - & 73.64 & 72.58 & 72.23 & 73.42 & 73.42 & \underline{73.91} & \cellcolor{lightgreen!50}\textbf{73.99} \\
 & $\text{OOD}_{\text{avg.}}$ & 71.98 & 71.52 & - & - & 71.53 & 71.50 & 71.16 & 71.75 & \underline{72.36} & 72.29 & \cellcolor{lightgreen!50}\textbf{72.40} \\
 & $\text{OOD}_{\text{worst}}$ & 69.80 & 69.44 & - & - & 68.99 & 69.61 & 69.10 & 69.69 & 70.19 & \underline{70.40} & \cellcolor{lightgreen!50}\textbf{70.61} \\
\hline
\multirow{3}{*}{\shortstack{Arxiv\\\citep{yao2022wild}}} 
 & $D_{S + 1}$ & 56.22 & 56.64 & - & - & 56.60 & 49.98 & 52.02 & 56.04 & 56.60 & \underline{56.66} & \cellcolor{lightgreen!50}\textbf{57.34} \\
 & $\text{OOD}_{\text{avg.}}$ & 52.43 & 52.95 & - & - & 52.78 & 47.27 & 48.91 & 52.07 & 53.15 & \underline{53.43} & \cellcolor{lightgreen!50}\textbf{54.20} \\
 & $\text{OOD}_{\text{worst}}$ & 49.37 & 49.97 & - & - & 49.73 & 44.77 & 46.03 & 48.97 & 50.19 & \underline{50.70} & \cellcolor{lightgreen!50}\textbf{51.41} \\
\hline
\multirow{3}{*}{\textbf{Overall Avg.}} 
 & $D_{S + 1}$ & \cellcolor{lightcyan!50}73.88 & \cellcolor{lightcyan!50}72.01 & \cellcolor{lightcyan!50}- & \cellcolor{lightcyan!50}- & \cellcolor{lightcyan!50}74.79 & \cellcolor{lightcyan!50}69.34 & \cellcolor{lightcyan!50}64.77 & \cellcolor{lightcyan!50}74.29 & \cellcolor{lightcyan!50}\underline{75.32} & \cellcolor{lightcyan!50}73.67 & \cellcolor{lightgreen!50}\textbf{76.26} \\
 & $\text{OOD}_{\text{avg.}}$ & \cellcolor{lightcyan!50}71.23 & \cellcolor{lightcyan!50}69.55 & \cellcolor{lightcyan!50}- & \cellcolor{lightcyan!50}- & \cellcolor{lightcyan!50}71.88 & \cellcolor{lightcyan!50}66.48 & \cellcolor{lightcyan!50}61.34 & \cellcolor{lightcyan!50}70.99 & \cellcolor{lightcyan!50}\underline{71.81} & \cellcolor{lightcyan!50}70.07 & \cellcolor{lightgreen!50}\textbf{72.45} \\
 & $\text{OOD}_{\text{worst}}$ & \cellcolor{lightcyan!50}67.59 & \cellcolor{lightcyan!50}65.46 & \cellcolor{lightcyan!50}- & \cellcolor{lightcyan!50}- & \cellcolor{lightcyan!50}67.84 & \cellcolor{lightcyan!50}63.40 & \cellcolor{lightcyan!50}57.09 & \cellcolor{lightcyan!50}67.94 & \cellcolor{lightcyan!50}\underline{68.63} & \cellcolor{lightcyan!50}66.67 & \cellcolor{lightgreen!50}\textbf{69.27} \\
\hline\hline
\end{tabular}
}
\vspace{-2mm}
\caption{Accuracy (\%) under CDGTD setting. Baselines include: ERM (IncERM), Mixup~\cite{zhang2018mixup}, SimCLR~\cite{chen2020simple}, SwAV~\cite{caron2020unsupervised}, EWC~\cite{kirkpatrick2017overcoming}, SI~\cite{Zenke2017ContinualLT}, A-GEM~\cite{Chaudhry2018EfficientLL}, DRAIN~\cite{Bai2022TemporalDG}, EvoS~\cite{xie2024evolving}, and W-Diff~\cite{xie2024weight}. Best and second-best results are \textbf{bolded} and \underline{underlined}. For FMoW, CLEAR-10/100, Huffpost and Arxiv, DRAIN is only applied to classifiers due to backbone size limitations. SimCLR and SwAV only apply to image benchmarks. ${\S}$ indicates TDG baselines.}
\label{tab:cdgtd_results}
\vspace{-4mm}
\end{table*}
\begin{table*}[t]
\small
\centering
\begin{tabular}{l|c|c|c|c|c|c|c|c|c}
\hline\hline\rule{0pt}{2ex}
\textbf{Method} & \textbf{Yearbook} & \textbf{RMNIST} & \textbf{C10} & \textbf{C100} & \textbf{FMoW} & \textbf{HuffPost} & \textbf{Arxiv} & \textbf{Overall} & \textbf{Rel. Cost} \\
\hline
\multicolumn{10}{c}{\cellcolor{gray!15} TDG setting}\\
ERM          & 0.02 & 0.02 & 0.30 & 1.58 & 2.34 & 3.05 & 7.92 & 2.18 & 1.00 \\
GI           & 0.21 & 1.31 & 0.32$^{*}$ & 3.54$^{*}$ & 5.35$^{*}$ & 3.87$^{*}$ & 9.75$^{*}$ & 3.48 & 12.01 \\
LSSAE        & 0.19 & 0.22 & 2.19 & 9.43 & 12.05 & - & - & - & 7.78 \\
TEA          & 0.04 & 0.04 & 0.33 & 1.62 & 2.43 & 3.23 & 8.57 & 2.32 & 1.33 \\
\multicolumn{10}{c}{\cellcolor{gray!15} CDGTD setting}\\
IncERM  & 0.02 & 0.02 & 0.30 & 1.58 & 2.34 & 3.03 & 7.95 & 2.18 & 1.00 \\
DRAIN        & 0.05 & 0.13 & 0.33$^{\dagger}$ & 1.75$^{\dagger}$ & 2.45$^{\dagger}$ & 3.07$^{\dagger}$ & 8.86 & 2.38 & 2.05 \\
EvoS         & 0.07 & 0.07 & 0.38 & 1.67 & 2.56 & 3.08 & 9.04 & 2.41 & 1.80 \\
W-Diff       & 3.12 & 6.74 & 3.47 & 32.35 & 65.31 & 13.18 & 77.93 & 28.87 & 81.01 \\
TEA          & 0.04 & 0.04 & 0.32 & 1.64 & 2.46 & 3.19 & 8.65 & 2.33 & 1.33 \\
\hline\hline
\end{tabular}

\vspace{-2mm}
\caption{Training cost (hours on A40 GPU) for each method. Rel. Cost is the computational cost ratio vs. ERM/IncERM, averaged across all tasks. C10 and C100 refer to CLEAR-10 and CLEAR-100~\cite{Lin2022TheCB} respectively. $^{*}$GI without finetuning. $^{\dagger}$Classifier-only DRAIN.}
\label{tab:training_cost}

\vspace{-2mm}
\end{table*}
\begin{table*}[t]
\centering
\small
\begin{tabular}{lcccccccc}
\hline\hline\rule{0pt}{2ex}
\footnotesize \textbf{Configuration} & 
\footnotesize \textbf{Yearbook} & 
\footnotesize \textbf{RMNIST} & 
\footnotesize \textbf{FMoW} & 
\footnotesize \textbf{C10} & 
\footnotesize \textbf{C100} & 
\footnotesize \textbf{Huffpost} & 
\footnotesize \textbf{Arxiv} & 
\footnotesize \textbf{Overall} \\
\hline
\multicolumn{9}{l}{\textbf{Single Model}} \\
\ -\ ERM & 88.46 & 92.14 & 59.76 & 81.20 & 63.19 & 71.50 & 52.38 & 72.66 \\
\ -\ \textit{Random Expert} & 87.46 & 82.29 & 59.37 & 81.63 & 63.26 & 71.34 & 52.05 & 71.06 \\
\ -\ \textit{Last Expert} & 95.42 & 92.17 & 60.49 & 81.53 & 63.16 & 71.33 & 54.57 & 74.10 \\
\multicolumn{9}{l}{\textbf{Weight Averaging}} \\
\ -\ \textit{Only Temporal Experts} & 95.41 & 92.64 & 60.54 & 82.12 & 66.32 & 71.43 & 53.79 & 74.61 \\
\ -\ \textit{Only Adaptive Averaging} & 94.03 & 92.92 & 60.83 & 83.07 & 66.85 & 71.73 & 53.26 & 74.67 \\
\textbf{Full TEA (ours)} & \textbf{95.95} & \textbf{94.47} & \textbf{62.45} & \textbf{83.16} & \textbf{66.96} & \textbf{72.12} & \textbf{55.23} & \textbf{75.76} \\
\hline\hline
\end{tabular}

\vspace{-2mm}
\caption{Ablation study of TEA components under the TDG setting with OOD average accuracy (\%). C10 and C100 refer to CLEAR-10 and CLEAR-100~\cite{Lin2022TheCB} respectively.}
\label{tab:tea_component_ab}
\vspace{-4mm}
\end{table*}

\begin{table}[t]
\centering
\small
\begin{tabular}{lcccc}
\hline\hline
\rule{0pt}{2ex}\footnotesize \textbf{Coeffs} & \footnotesize \textbf{Yearbook} & \footnotesize \textbf{RMNIST} & \footnotesize \textbf{Arxiv} & \footnotesize \textbf{Overall} \\

\midrule
\textcolor{gray}{ERM}     & \textcolor{gray}{88.46} & \textcolor{gray}{92.14} & \textcolor{gray}{52.38} & \textcolor{gray}{72.66} \\
Correct   & 95.95 & 94.47 & 55.23 & 75.76 \\
Reversed  & 82.95 & 77.03 & 50.13 & 70.05 \\
\hline\hline
\end{tabular}
\vspace{-2mm}
\caption{Sanity check with correct and reversed coefficients. Overall is averaged across the 7 benchmarks.}
\label{tab:temporal_sensitivity}
\vspace{-5mm}
\end{table}

\section{Experimental Results}
\label{sec:exp}

We first introduce the major experimental setups, with detailed configurations provided in Appendix~\ref{sec:exp_details}. For fair and consistent comparisons, we follow the configurations from \citet{xie2024evolving,xie2024weight} for all benchmarks, except for CLEAR-10\&100 which are not covered in these works.

\smallskip
\noindent\textbf{Benchmarks.} We include Huffpost, Arxiv, Yearbook and FMoW from \citet{yao2022wild}, CLEAR-10/100 from \citet{Lin2022TheCB}, and Rotated MNIST (RMNIST). Huffpost and Arxiv are text benchmarks; others are image benchmarks. RMNIST and Yearbook are small-scale; others are large in comparison. Each dataset is divided into first $S$ source and last $F$ target domains with ratios $S:F$ of: Yearbook (16:5), RMNIST (6:3), FMoW (13:3), Huffpost (4:3), Arxiv (9:7), and CLEAR-10/100 (5:5). Each source domain uses a random 90\%-10\% train-validation split. 

\smallskip
\noindent\textbf{Model Architectures.} We use: 4-layer CNN for Yearbook, ConvNet~\citep{qin2022generalizing} for RMNIST, DenseNet-121~\citep{huang2017densely} for FMoW, DistilBERT~\citep{sanh2019distilbert} for Huffpost/Arxiv, and ResNet-18/50~\citep{he2016deep} for CLEAR-10/100.

\smallskip
\noindent\textbf{Baselines:} For TDG, we evaluate against ERM, IRM~\cite{arjovsky2019invariant}, CORAL~\cite{Sun2016DeepCC}, Mixup~\cite{zhang2018mixup}, LISA~\cite{yao2022improving}, GI~\cite{nasery2021training}, LSSAE~\cite{qin2022generalizing}, SWAD~\cite{cha2021swad}, and DiWA~\cite{rame2022diwa}, where GI and LSSAE are representative TDG methods and SWAD and DiWA are representative weight averaging approaches. For CDGTD, we include Incremental ERM (IncERM), Mixup~\cite{zhang2018mixup}, SimCLR~\cite{chen2020simple}, SwAV~\cite{caron2020unsupervised}, EWC~\cite{kirkpatrick2017overcoming}, SI~\cite{Zenke2017ContinualLT}, A-GEM~\cite{Chaudhry2018EfficientLL}, DRAIN~\cite{Bai2022TemporalDG}, EvoS~\cite{xie2024evolving}, and W-Diff~\cite{xie2024weight}, with EvoS and W-Diff being state-of-the-art CDGTD methods. Due to computational constraints (\eg GI finetuning costs 400 GPU hours per epoch), full GI and DRAIN are applied only on Yearbook and RMNIST. For larger benchmarks, we use GI without finetuning and apply DRAIN only to the classifier.

\smallskip
\noindent \textbf{Method Configurations.} TEA maintain equivalent total training steps (e.g., 25 baseline epochs = 20 pretraining + 5 finetuning for TEA). Other TEA and baseline details are in Appendix~\ref{sec:method_config}.

\subsection{Results}
\label{sec:main_res}

\textbf{TDG setting} results and comparisons are presented in Table~\ref{tab:tdg_results}. Our TEA outperforms all baseline methods on both image and text benchmarks. Specifically, we observe: a). Prior TDG baselines (GI~\cite{nasery2021training} and LSSAE~\cite{qin2022generalizing}) perform well on small-scale benchmarks (RMNIST and Yearbook~\cite{yao2022wild}) but degrade significantly on other large-scale benchmarks~\cite{Lin2022TheCB,yao2022wild}. While GI's poor performance potentially stems from computational constraints preventing finetuning stage, LSSAE was fully applied, indicating that prior TDG methods also struggle to model temporal distribution shifts on large-scale tasks beyond computational limitations. In contrast, TEA consistently improves performance across all scales, outperforming GI by up to 30\% and LSSAE by up to 69\%; b). TEA also consistently outperforms weight averaging methods (DiWA~\cite{rame2022diwa} and SWAD~\cite{cha2021swad}), validating that our approach not only benefits from sampling experts with functional diversity and parameter similarity but further leverages adaptive averaging coefficients to specifically address temporal shifts, thereby enhancing temporal generalization beyond standard weight averaging techniques.

\smallskip
\noindent\textbf{CDGTD setting} results and comparisons are presented in Table~\ref{tab:cdgtd_results}. Our TEA still achieves the best performance on average, outperforming state-of-the-art CDGTD baselines, EvoS~\cite{xie2024evolving} and W-Diff~\cite{xie2024weight}. On text benchmarks, our TEA consistently performs the best, while on image benchmarks, although different benchmarks favor different methods, our TEA generally ranks within the top two. These results demonstrate the superiority and flexibility of TEA, showing that TEA can effectively improve temporal generalization even under imited data access constraints.

\smallskip
\noindent\textbf{Training Cost Analysis} is presented in Table~\ref{tab:training_cost}. Early TDG methods (GI~\cite{nasery2021training}, LSSAE~\cite{qin2022generalizing}, and DRAIN~\cite{Bai2022TemporalDG}) significantly increase training costs (see Yearbook and RMNIST for full costs). Even classifier-only W-Diff averages 81× the training cost. In contrast, our TEA only slightly increases cost by 33\% over ERM in both TDG and CDGTD, being up to 60x more efficient.

\subsection{Ablation Study and Analysis}
\label{sec:ablation_study}

\textbf{Single Model Ablation} results are shown in Table~\ref{tab:tea_component_ab}. The \textit{Random Expert} average accuracies from randomly selected temporal experts, while \textit{Last Expert} shows accuracies from the last domain experts. \textit{Random Expert} performs worse than ERM, indicating that our method does not simply improve domain-agnostic convergence during finetuning. \textit{Last Expert} outperforms \textit{Random Expert}, demonstrating that our temporal finetuning enables the model to learn domain-specific distributions, achieving functional diversity among experts.

\smallskip
\noindent\textbf{Weight Averaging Ablation} are shown in Table~\ref{tab:tea_component_ab}. Recall that TEA optimizes two tradeoffs: (1) functional diversity vs.\ parameter similarity (with temporal experts), and (2) bias vs.\ variance (with adaptive averaging). \textit{Only Temporal Experts} uses uniform coefficients to average experts, optimizing only tradeoff 1, while \textit{Only Adaptive Averaging} samples domain-agnostic weights then trains Time2Vec~\cite{kazemi2019time2vec} for adaptive coefficients (detailed in Appendix~\ref{sec:ablation_details}), optimizing only tradeoff 2. Both variants underperform full TEA, validating the necessity of both design choices.

\begin{figure*}[t]
    \centering
    \begin{subfigure}[b]{0.38\textwidth}        \includegraphics[width=\linewidth]{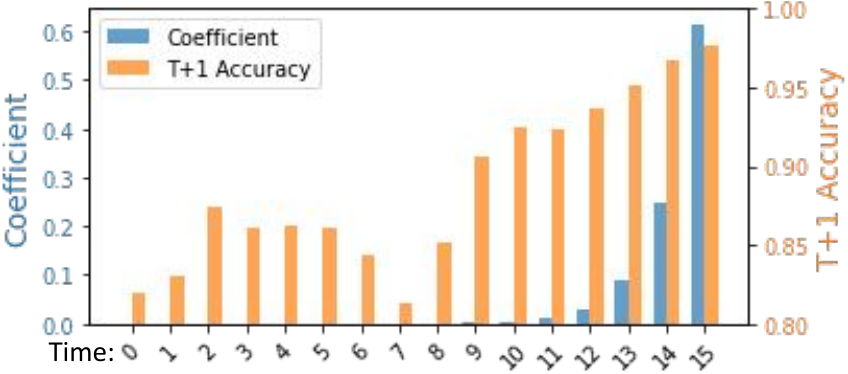}
        \caption{Yearbook.}
        \label{fig:yearbook_coeffs}
    \end{subfigure}
    \begin{subfigure}[b]{0.38\textwidth}
    \includegraphics[width=\linewidth]{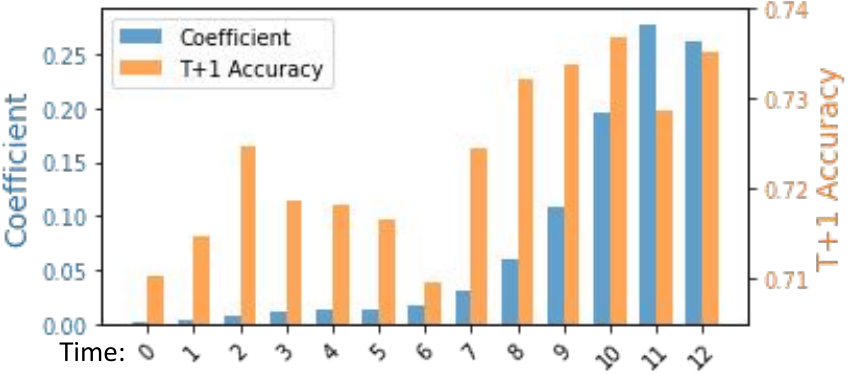}
        \caption{FMoW.}
        \label{fig:fmow_coeffs}
    \end{subfigure}
    \vspace{-2mm}
    \caption{Visualization of averaging coefficients and accuracies of experts on target domain $D_{S+1}$.}
    \label{fig:coeff_ablation}
    \vspace{-4mm}
\end{figure*}
\smallskip
\noindent\textbf{Temporal Sanity Check} are shown in Figure~\ref{fig:coeff_ablation} and Table~\ref{tab:temporal_sensitivity}. Our adaptive averaging should increase coefficients for better-performing experts on future domains while decreasing coefficients for poor performers. Figure~\ref{fig:coeff_ablation} confirms this design by showing higher coefficients for higher-performing models on domain $D_{S+1}$. Table~\ref{tab:temporal_sensitivity} validates our design by showing that reversing coefficient order leads to worse OOD accuracy than ERM.



\begin{table}[t]
\small
\centering
\begin{tabular}{l|c|cc}
\hline\hline
\rule{0pt}{2ex}\textbf{Strategy} & \textbf{Ratio} & \textbf{Yearbook} & \textbf{FMoW} \\
\hline

\multirow{3}{*}{Last} 
  & 25\%  & 95.45 & 60.88 \\
  & 50\%  & 95.34 & 61.66 \\
  & 75\%  & 95.46 & 62.26 \\
\hline
\multirow{3}{*}{Uniform} 
  & 25\%  & 95.18 & 60.76 \\
  & 50\%  & 95.35 & 61.26 \\
  & 75\%  & 95.50 & 61.69 \\
\hline\hline
\rule{0pt}{2ex}All     & 100\% & 95.95 & 62.45 \\
\hline\hline
\end{tabular}
\vspace{-2mm}
\caption{Ablation study on partial expert fine-tuning accuracy (\%) across domain selection strategies (Last, Uniform, All) at different ratios (25\%, 50\%, 75\%). We report $\text{OOD}_{\text{avg}}$ for each dataset. Fine-tuning on recent domains performs comparably to all domains.} 
\label{tab:tea_domain_ratios}
\vspace{-5mm}
\end{table}

\smallskip
\noindent\textbf{Partial Expert \& Fine-tuning Ablation.} We investigate TEA's scalability by fine-tuning on domain subsets using two strategies: Last (recent domains closest to future) and Uniform (uniform sampling). Table~\ref{tab:tea_domain_ratios} shows results with 25\%, 50\%, and 75\% domain ratios on Yearbook and FMoW~\citep{yao2022wild} under the TDG setting. Performance improves with more domains, and selecting recent domains significantly outperforms uniform selection with limited experts, confirming temporally closer domains are more crucial and suggesting fine-tuning only on recent domains could maintain performance while reducing costs.

\begin{table}[t]
\centering
\small
\begin{tabular}{c|c c c c}
\hline\hline
\rule{0pt}{2ex}\textbf{Buffer} & \textbf{RMNIST} & \textbf{Yearbook} & \textbf{FMoW} & \textbf{Arxiv} \\
\hline
1\%  & 92.99 & 92.78 & 53.20 & 52.17 \\
3\%  & 93.37 & 94.43 & 53.76 & 53.76 \\
5\%  & 93.51 & 95.15 & 54.63 & 54.10 \\
10\% & 93.77 & 95.20 & 55.21 & 54.20 \\
20\% & 92.71 & 95.47 & 55.37 & 54.43 \\
\hline\hline
\end{tabular}
\vspace{-2mm}
\caption{Ablation study on memory buffer size in CDGTD setting across different buffer sizes (1\%, 3\%, 5\%, 10\%, 20\%). $\text{OOD}_{\text{avg}}$ accuracy (\%) is reported. 5\%-10\% buffer sizes achieve good balance between memory efficiency and generalization performance.}
\label{tab:buffer_results}
\vspace{-2mm}
\end{table}

\smallskip
\noindent\textbf{Ablation Study on Memory Buffer Size.} We evaluate five buffer sizes (1\%, 3\%, 5\%, 10\%, and 20\%) in the CDGTD setting. Table~\ref{tab:buffer_results} shows that 1\% and 3\% buffer exhibits lower performance, while 5\% closely approaches 10\% performance. Expanding to 20\% yields only marginal improvements. Results demonstrate that 5\%--10\% buffer sizes provide optimal balance between memory efficiency and performance, justifying our 10\% choice.


\begin{table}[t]
\small
\centering
\begin{tabular}{c|c|c|c c}
\hline\hline
\rule{0pt}{2ex}\textbf{Method} & \textbf{Buffer} & \textbf{Domain} & \textbf{Yearbook} & \textbf{FMoW} \\
\hline
\rule{0pt}{2ex}IncERM & --   & --     & 94.72 & 53.99 \\
\hline
\multirow{3}{*}{TEA} 
  & 10\% & All    & 95.20 & 55.21 \\
  & 5\%  & Last 5 & 95.03 & 54.48 \\
  & 5\%  & Last 3 & 94.94 & 54.38 \\

\hline\hline
\end{tabular}
\vspace{-2mm}
\caption{Joint reduction of buffer and domain ablation in CDGTD setting. $\text{OOD}_{\text{avg}}$ accuracy (\%) is reported. 5\% buffer size with the last 3-5 domains achieves comparable performance to a 10\% buffer across all domains while providing up to 11$\times$ memory reduction.}
\label{tab:tea_buffer_results}
\vspace{-5mm}
\end{table}

\smallskip
\noindent\textbf{Joint Reduction of Buffer and Domain.} Table~\ref{tab:tea_buffer_results} shows that we can further compress total buffer size in CDGTD setting through joint buffer and domain reduction. 5\% buffer with last 3-5 domains achieves comparable performance to 10\% buffer across all domains with substantial generalization gains over IncERM, while reducing the buffer size requirements by up to 11×.


\begin{table}[t]
\small
\centering
\begin{tabular}{c c c c}
\hline\hline
\rule{0pt}{2ex}\textbf{Method} & \textbf{Distribution} & \textbf{Yearbook} & \textbf{FMoW} \\
\hline
ERM   & Both & 88.46 & 59.76 \\
DiWA  & Both & 95.36 & 60.77 \\
\hline
\rule{0pt}{2ex}TEA-S & Abrupt   & 95.47 & 60.85 \\
TEA   & Smooth & 95.95 & 62.45 \\
\hline\hline
\end{tabular}
\vspace{-2mm}
\caption{Ablation study on abrupt distribution shifts: shuffled vs. unshuffled source domains. $\text{OOD}_{\text{avg}}$ accuracy (\%) is reported. TEA maintains generalization gains despite abrupt distribution changes and remains comparable to strong DG baselines like DiWA. ERM and DiWA are unaffected by smooth vs. abrupt shifts.}
\label{tab:abrupt_shifts}
\vspace{-2mm}
\end{table}

\smallskip
\noindent\textbf{Abrupt and Unpredictable Shift Ablation.} Table~\ref{tab:abrupt_shifts} shows results using shuffled source domains to simulate abrupt distribution shifts on Yearbook and FMoW. Theoretically, under abrupt shifts, TDG degrades to standard DG where TEA's adaptive averaging gains should diminish while model averaging effects remain. Results align with theory: TEA with shuffled domains (TEA-S) performs worse than standard TEA but remains comparable with strong DG baseline DiWA, demonstrating robustness when temporal assumptions are violated.


\begin{table}[t]
\small
\centering
\begin{tabular}{c c c c}
\hline\hline
\rule{0pt}{2ex}\textbf{Setting} & \textbf{Order} & \textbf{Yearbook} & \textbf{FMoW} \\
\hline
\multirow{2}{*}{TDG} 
\rule{0pt}{2ex}  & Chronological & 95.46 & 61.31 \\
  & Reverse       & 95.95 & 62.45 \\
\hline
\multirow{2}{*}{CDGTD} 
  & Chronological & 95.17 & 54.68 \\
  & Reverse       & 95.29 & 55.21 \\
\hline\hline
\end{tabular}
\vspace{-2mm}
\caption{Ablation study on fine-tuning order comparing chronological vs. reverse strategies. $\text{OOD}_{\text{avg}}$ accuracy (\%) is reported. Reverse fine-tuning consistently outperforms across TDG and CDGTD.}
\label{tab:finetuning_order}
\vspace{-5mm}
\end{table}

\smallskip
\noindent\textbf{Fine-tuning Order Ablation.} We compare chronological and reverse fine-tuning orders in both TDG and CDGTD settings. Table~\ref{tab:finetuning_order} shows that reverse fine-tuning consistently outperforms chronological fine-tuning by up to 2\% on Yearbook and FMoW. The advantage becomes more pronounced on large-scale benchmarks like FMoW, justifying our reverse fine-tuning strategy where experts are trained from most recent to earliest domains.
\section{Conclusion}
\label{sec:conclusion}
In this work, we propose Temporal Expert Averaging (TEA), an efficient weight averaging framework for scaled-up Temporal Domain Generalization (TDG) tasks.
Based on theoretical insights, TEA uses constrained temporal finetuning to create functionally diverse yet parameter-similar experts, then adaptively averages them using coefficients derived from temporal dynamics of weight deviation principal components. Comprehensive evaluation demonstrates TEA's superior performance and efficiency across TDG and CDGTD settings. Since prior TDG work focuses on small-scale scenarios, we hope this encourages research on large-scale temporal generalization.
\smallskip

\noindent\textbf{Acknowledgments.} This material is based upon work supported, in part, by a grant from Meta.

\section{Limitations} 

Like prior TDG methods, our TEA relies on smooth distribution shift assumptions and cannot guarantee performance with abrupt shifts. Additionally, since most large-scale TDG benchmarks use discrete domains, we only explored discrete settings, though TEA could theoretically extend to Continuous Temporal Domain Generalization (CTDG).

However, the smooth distribution shift assumption already provides TEA and TDG methods with broad application potential, as many real-world phenomena exhibit gradual temporal changes, including language semantic evolution, object appearance variations, and fashion trend shifts. On the other hand, when this assumption is violated, the task setting reverts to standard Domain Generalization (DG), and our ablation studies show that TEA maintains considerable generalization gains even under abrupt distribution shifts, demonstrating robustness beyond its theoretical assumptions.

TEA can theoretically accommodate CTDG settings by utilizing continuous timestamps rather than discrete temporal labels. However, we are constrained by the lack of large-scale CTDG datasets, as existing continuous temporal datasets remain small-scale while our focus is on scaling up TDG for large-scale scenarios. We will explore this direction in future work when large-scale continuous temporal datasets become available.

\bibliography{custom}

\newpage
\appendix

\section{Experimental Setup Details}
\label{sec:exp_details}
\subsection{Benchmark Introduction}

\noindent\textbf{Huffpost}~\cite{ginosar2015century} is a text classification benchmark comprising news headlines from The Huffington Post spanning 2012-2018. The task requires classifying headlines into 11 news categories: ``Black Voices'', ``Business'', ``Comedy'', ``Crime'', ``Entertainment'', ``Impact'', ``Queer Voices'', ``Science'', ``Sports'', ``Tech'', and ``Travel''. This temporal dataset captures evolving journalistic styles and content trends in digital media over six years. We adopt a temporal split using the first 4 years as training domains and the final 3 years as test domains for evaluating temporal generalization. Sample distributions across domains are detailed in Table~\ref{tab:huffpost_stat}.

\noindent\textbf{Arxiv}~\cite{ginosar2015century} is a text classification benchmark containing paper titles and their corresponding primary categories spanning 2007-2022. The task requires classifying research papers into one of 172 categories based solely on their titles. This temporal dataset reflects the dynamic evolution of research fields, with changing academic trends and emerging disciplines captured across the 16-year timespan. We adopt a temporal split using the first 9 years as training domains and the final 7 years as test domains for evaluating temporal generalization. Sample distributions across domains are presented in Table~\ref{tab:arxiv_subset}.

\smallskip
\noindent\textbf{Yearbook} dataset, sourced from \citet{yao2022wild} and built upon the MIT-licensed Portraits dataset~\cite{ginosar2015century}, comprises 32×32 grayscale yearbook portraits from 128 American high schools across 27 states. Spanning eight decades (1930-2013), this temporal dataset captures the evolution of fashion trends and societal changes, making it particularly suitable for evaluating algorithmic performance on temporal domain shift. We formulate the task as binary gender classification, partitioning the timeline into 4-year intervals to create 21 distinct domains. Following standard practice, we allocate the initial 16 domains for training and reserve the final 5 domains for out-of-domain evaluation. Sample distributions across domains are detailed in Table~\ref{tab:yearbook_stat}.

\smallskip
\noindent\textbf{Rotated MNIST} (RMNIST) derives from the classic MNIST dataset~\cite{Deng2012TheMD} by systematically applying rotational transformations from 0° to 80° in 10° increments, creating 9 sequential domains that simulate temporal distribution shift. This benchmark evaluates 10-class digit classification performance on 28×28 grayscale images under gradually increasing rotational distortion. We adopt a 6-3 domain split, utilizing the initial six domains for model training and evaluating generalization on the final three domains.

\smallskip
\noindent\textbf{FMoW}~\cite{ginosar2015century} contains 224×224 RGB satellite imagery spanning 2002-2017 across 200 countries. This temporal benchmark captures natural evolution in visual features driven by human development and environmental changes over time. The classification task involves predicting functional land use across 62 categories, ranging from residential areas to industrial facilities. We partition the dataset temporally with each year constituting a distinct domain, yielding 16 total domains. Training utilizes the first 13 domains, while the final 3 domains serve as out-of-distribution test sets. Domain-wise sample distributions are provided in Table~\ref{tab:fmow_stat}.

\smallskip
\noindent\textbf{CLEAR-10\&100}~\cite{Lin2022TheCB} contain user-uploaded images from 2007-2014 with natural temporal shifts of visual concepts. Samples are organized into 10 chronologically ordered domains. CLEAR-10 comprises 10 classes with 3,000 samples per domain (300 per class), while CLEAR-100 contains 100 classes with 10,000 samples per domain (100 per class). We set the image input shape as $(224, 224, 3)$ and use the first 5 domains as source domains and the final 5 domains as target domains for temporal generalization evaluation.

\begin{table}[h]
\centering
\resizebox{0.48\textwidth}{!}{
\begin{tabular}{ccccc}
\toprule
\textbf{Domain} & \textbf{Year} & \textbf{Training Split} & \textbf{Validation Split} & \textbf{All} \\
\midrule
1 & 2012 & 6701  & 744  & 7446 \\
2 & 2013 & 7492  & 832  & 8325 \\
3 & 2014 & 9539  & 1059 & 10599 \\
4 & 2015 & 11826 & 1313 & 13140 \\
5 & 2016 & 10548 & 1172 & 11721 \\
6 & 2017 & 7907  & 878  & 8786 \\
7 & 2018 & 3501  & 388  & 3890 \\
\midrule
\textbf{Total} & 2012--2018 & \textbf{57514} & \textbf{6386} & \textbf{63907} \\
\bottomrule
\end{tabular}
}
\caption{Domain Sizes for Huffpost~\cite{yao2022wild}}
\label{tab:huffpost_stat}
\end{table}

\begin{table}[h]
\centering
\resizebox{0.48\textwidth}{!}{
\begin{tabular}{ccccc}
\toprule
\textbf{Domain} & \textbf{Year} & \textbf{Training Split} & \textbf{Validation Split} & \textbf{All} \\
\midrule
1  & 2007 & 131550 & 14616  & 146167 \\
2  & 2008 & 62460  & 6939   & 69400 \\
3  & 2009 & 206244 & 22916  & 229161 \\
4  & 2010 & 50665  & 5629   & 56295 \\
5  & 2011 & 55741  & 6193   & 61935 \\
6  & 2012 & 51678  & 5741   & 57420 \\
7  & 2013 & 64951  & 7216   & 72168 \\
8  & 2014 & 79498  & 8833   & 88332 \\
9  & 2015 & 193979 & 21553  & 215533 \\
10 & 2016 & 120682 & 13409  & 134092 \\
11 & 2017 & 111024 & 12336  & 123361 \\
12 & 2018 & 123891 & 13765  & 137657 \\
13 & 2019 & 142767 & 15862  & 158630 \\
14 & 2020 & 166014 & 18445  & 184460 \\
15 & 2021 & 201241 & 22360  & 223602 \\
16 & 2022 & 89765  & 9973   & 99739 \\
\midrule
\textbf{Total} & 2007--2022 & \textbf{1852150} & \textbf{205786} & \textbf{2057952} \\
\bottomrule
\end{tabular}
}
\caption{Domain Size for Arxiv~\cite{yao2022wild}}
\label{tab:arxiv_subset}
\end{table}

\begin{table}[h]
\centering
\resizebox{0.48\textwidth}{!}{
\begin{tabular}{cccccc}
\toprule
\textbf{Domain} & \textbf{Interval} & \textbf{Training Split} & \textbf{Validation Split} & \textbf{All} \\
\midrule
1  & 1930 -- 1933 & 758  & 87  & 845 \\
2  & 1934 -- 1937 & 1149 & 130 & 1279 \\
3  & 1938 -- 1941 & 949  & 108 & 1057 \\
4  & 1942 -- 1945 & 2353 & 263 & 2616 \\
5  & 1946 -- 1949 & 1229 & 138 & 1367 \\
6  & 1950 -- 1953 & 1082 & 122 & 1204 \\
7  & 1954 -- 1957 & 1646 & 185 & 1831 \\
8  & 1958 -- 1961 & 1295 & 146 & 1441 \\
9  & 1962 -- 1965 & 1468 & 166 & 1634 \\
10 & 1966 -- 1969 & 2227 & 249 & 2476 \\
11 & 1970 -- 1973 & 1634 & 183 & 1817 \\
12 & 1974 -- 1977 & 2238 & 250 & 2488 \\
13 & 1978 -- 1981 & 1553 & 175 & 1728 \\
14 & 1982 -- 1985 & 2331 & 261 & 2592 \\
15 & 1986 -- 1989 & 1792 & 201 & 1993 \\
16 & 1990 -- 1993 & 1729 & 195 & 1924 \\
17 & 1994 -- 1997 & 1882 & 211 & 2093 \\
18 & 1998 -- 2001 & 2136 & 239 & 2375 \\
19 & 2002 -- 2005 & 1868 & 210 & 2078 \\
20 & 2006 -- 2009 & 1010 & 114 & 1124 \\
21 & 2010 -- 2013 & 1102 & 125 & 1227 \\
\midrule
\textbf{Total} & 1930 -- 2013 & \textbf{33431} & \textbf{3758} & \textbf{37189} \\
\bottomrule
\end{tabular}
}
\caption{Domain Sizes for Yearbook~\cite{yao2022wild}}
\label{tab:yearbook_stat}
\end{table}

\begin{table}[h]
\centering
\resizebox{0.48\textwidth}{!}{
\begin{tabular}{ccccc}
\toprule
\textbf{Domain} & \textbf{Year} & \textbf{Training Split} & \textbf{Validation Split} & \textbf{All} \\
\midrule
1  & 2002 & 1676  & 227  & 1903 \\
2  & 2003 & 2279  & 276  & 2555 \\
3  & 2004 & 1755  & 240  & 1995 \\
4  & 2005 & 2512  & 324  & 2836 \\
5  & 2006 & 3155  & 406  & 3561 \\
6  & 2007 & 1497  & 190  & 1687 \\
7  & 2008 & 2261  & 298  & 2559 \\
8  & 2009 & 7439  & 935  & 8374 \\
9  & 2010 & 18957 & 2456 & 21413 \\
10 & 2011 & 22111 & 2837 & 24948 \\
11 & 2012 & 24704 & 3138 & 27842 \\
12 & 2013 & 3465  & 385  & 3850 \\
13 & 2014 & 5572  & 620  & 6192 \\
14 & 2015 & 8885  & 988  & 9873 \\
15 & 2016 & 14363 & 1596 & 15959 \\
16 & 2017 & 5534  & 615  & 6149 \\
\midrule
\textbf{Total} & 2002--2017 & \textbf{126165} & \textbf{15531} & \textbf{141696} \\
\bottomrule
\end{tabular}
}
\caption{Domain Sizes for FMoW~\cite{yao2022wild}}
\label{tab:fmow_stat}
\end{table}

\subsection{Method Configurations}
\label{sec:method_config}

\noindent\textbf{Huffpost}~\cite{yao2022wild} uses pretrained DistilBERT base model~\cite{sanh2019distilbert} as the backbone. All baseline methods are trained on 90\% randomly split training data from source domains for 50 epochs with learning rate 2e-5 (except A-GEM which uses 1e-7). Other baseline configurations follow~\citet{xie2024evolving,xie2024weight}.

\smallskip
\noindent\textbf{TEA for Huffpost} uses the same DistilBERT backbone. Under TDG setting, TEA first trains on all source domain training splits using ERM for 45 epochs with learning rate 2e-5 during the pretraining stage, then performs temporal finetuning for 5 epochs on each domain in reverse temporal order (from 2015 to 2012) using SI with learning rate 5e-6 and constraint strength $c_{si} = 0.1$. Under CDGTD setting, we adopt 47-epoch incremental ERM training on each domain (from 2012 to 2015) with learning rate 2e-5, followed by 30 temporal finetuning epochs on each domain in reverse temporal order (from 2015 to 2012) using SI with learning rate 5e-6 and constraint strength $c_{si} = 0.1$. Note that temporal finetuning under CDGTD uses only 10\% of the data, so the total training cost remains 47+30×0.1=50 epochs. During temporal finetuning on each domain, we sample model weights at $K=5$ evenly spaced training steps and uniformly average them to obtain expert model weights. For PCA on expert deviations, we use the top 10 principal components. For ARIMA estimation, we employ an ARIMA(1,1,1) model. When computing averaging coefficients, we set the sharpness hyperparameter $r=5$.

\noindent\textbf{Arxiv}~\cite{yao2022wild} also uses pretrained DistilBERT base model~\cite{sanh2019distilbert} as the backbone. All baseline methods are trained on 90\% randomly split training data from source domains for 5 epochs with learning rate 2e-5 (except A-GEM which uses 1e-6). Other baseline configurations follow~\citet{xie2024evolving,xie2024weight}.

\smallskip
\noindent\textbf{TEA for Arxiv} uses the same DistilBERT backbone. Under TDG setting, TEA first trains on all source domain training splits using ERM for 4 epochs with learning rate 2e-5 during the pretraining stage, then performs temporal finetuning for 1 epoch on each domain in reverse temporal order (from 2015 to 2007) using SI with learning rate 5e-6 and constraint strength $c_{si} = 0.1$. Under CDGTD setting, we adopt 4-epoch incremental ERM training on each domain (from 2007 to 2015) with learning rate 2e-5, followed by 10 temporal finetuning epochs on each domain in reverse temporal order (from 2015 to 2007) using SI with learning rate 5e-6 and constraint strength $c_{si} = 0.1$. Note that temporal finetuning under CDGTD uses only 10\% of the data, so the total training cost remains 4+10×0.1=5 epochs. During temporal finetuning on each domain, we sample model weights at $K=5$ evenly spaced training steps and uniformly average them to obtain expert model weights. For PCA on expert deviations, we use the top 10 principal components. For ARIMA estimation, we employ an ARIMA(1,1,1) model. When computing averaging coefficients, we set the sharpness hyperparameter $r=5$.

\smallskip
\noindent\textbf{Yearbook}~\cite{yao2022wild} uses a 4-layer convolutional network from \citet{yao2022wild}. All baseline methods are trained on 90\% randomly split training data from source domains for 50 epochs with learning rate 1e-3. Other baseline configurations follow~\citet{xie2024evolving,xie2024weight}.

\smallskip
\noindent\textbf{TEA for Yearbook} uses the same 4-layer convolutional network from \citet{yao2022wild}. Under TDG setting, TEA first trains on all source domain training splits using ERM for 40 epochs with learning rate 1e-3 during the pretraining stage, then performs temporal finetuning for 10 epochs on each domain in reverse temporal order (from D$_{16}$ to D$_1$) using SI with learning rate 5e-4 and constraint strength $c_{si} = 0.1$. Under CDGTD setting, we adopt 48-epoch incremental ERM training on each domain (from D$_1$ to D$_{16}$) with learning rate 1e-3, followed by 20 temporal finetuning epochs on each domain in reverse temporal order (from D$_{16}$ to D$_1$) using SI with learning rate 5e-4 and constraint strength $c_{si} = 0.1$. Note that temporal finetuning under CDGTD uses only 10\% of the data, so the total training cost remains 48+20×0.1=50 epochs. During temporal finetuning on each domain, we sample model weights at $K=5$ evenly spaced training steps and uniformly average them to obtain expert model weights. For PCA on expert deviations, we use the top 10 principal components. For ARIMA estimation, we employ an ARIMA(1,1,1) model. When computing averaging coefficients, we set the sharpness hyperparameter $r=5$.

\smallskip
\noindent\textbf{RMNIST} adopts the ConvNet in \citet{qin2022generalizing}. All baseline methods are trained on 90\% randomly split training data from source domains for 50 epochs with learning rate 1e-3 (except A-GEM which uses 1e-5). Other baseline configurations follow~\citet{xie2024evolving,xie2024weight}.

\smallskip
\noindent\textbf{TEA for RMNIST} uses the same ConvNet. Under TDG setting, TEA first trains on all source domain training splits using ERM for 40 epochs with learning rate 1e-3 during the pretraining stage, then performs temporal finetuning for 10 epochs on each domain in reverse temporal order (from D$_{6}$ to D$_1$) using SI with learning rate 2e-4 and constraint strength $c_{si} = 0.1$. Under CDGTD setting, we adopt 48-epoch incremental ERM training on each domain (from D$_1$ to D$_{6}$) with learning rate 1e-3, followed by 20 temporal finetuning epochs on each domain in reverse temporal order (from D$_{6}$ to D$_1$) using SI with learning rate 2e-4 and constraint strength $c_{si} = 0.1$. Note that temporal finetuning under CDGTD uses only 10\% of the data, so the total training cost remains 48+20×0.1=50 epochs. During temporal finetuning on each domain, we sample model weights at $K=5$ evenly spaced training steps and uniformly average them to obtain expert model weights. For PCA on expert deviations, we use the top 10 principal components. For ARIMA estimation, we employ an ARIMA(1,1,1) model. When computing averaging coefficients, we set the sharpness hyperparameter $r=5$.

\smallskip
\noindent\textbf{FMoW}~\cite{yao2022wild} adopts a DenseNet-121~\cite{huang2017densely} backbone pretrained on ImageNet~\cite{deng2009imagenet}. All baseline methods are trained on 90\% randomly split training data from source domains for 25 epochs with learning rate 2e-4 (except A-GEM which uses 1e-6). Other baseline configurations follow~\citet{xie2024evolving,xie2024weight}.

\smallskip
\noindent\textbf{TEA for FMoW} uses the same DenseNet-121~\cite{huang2017densely}. Under TDG setting, TEA first trains on all source domain training splits using ERM for 20 epochs with learning rate 2e-4 during the pretraining stage, then performs temporal finetuning for 5 epochs on each domain in reverse temporal order (from D$_{13}$ to D$_1$) using SI with learning rate 7e-5 and constraint strength $c_{si} = 0.1$. Under CDGTD setting, we adopt 23-epoch incremental ERM training on each domain (from D$_1$ to D$_{13}$) with learning rate 2e-4, followed by 20 temporal finetuning epochs on each domain in reverse temporal order (from D$_{13}$ to D$_1$) using SI with learning rate 2e-5 and constraint strength $c_{si} = 0.1$. Note that temporal finetuning under CDGTD uses only 10\% of the data, so the total training cost remains 23+20×0.1=25 epochs. During temporal finetuning on each domain, we sample model weights at $K=5$ evenly spaced training steps and uniformly average them to obtain expert model weights. For PCA on expert deviations, we use the top 10 principal components. For ARIMA estimation, we employ an ARIMA(1,1,1) model. When computing averaging coefficients, we set the sharpness hyperparameter $r=1$.

\smallskip
\noindent\textbf{CLEAR-10}~\cite{Lin2022TheCB} adopts a ResNet-18~\cite{he2016deep}. All baseline methods are trained on 90\% randomly split training data from source domains for 50 epochs with batch size 128 and learning rate 1e-3 (except A-GEM which uses 1e-6). Other baseline configurations follow the FMoW configurations from \citet{xie2024evolving,xie2024weight}.

\smallskip
\noindent\textbf{TEA for CLEAR-10} uses the same ResNet-18~\cite{he2016deep}. Batch size is 128. Under TDG setting, TEA first trains on all source domain training splits using ERM for 45 epochs with learning rate 1e-3 during the pretraining stage, then performs temporal finetuning for 5 epochs on each domain in reverse temporal order (from D$_{5}$ to D$_1$) using SI with learning rate 1e-4 and constraint strength $c_{si} = 0.1$. Under CDGTD setting, we adopt 49-epoch incremental ERM training on each domain (from D$_1$ to D$_{5}$) with learning rate 1e-3, followed by 10 temporal finetuning epochs on each domain in reverse temporal order (from D$_{5}$ to D$_1$) using SI with learning rate 1e-4 and constraint strength $c_{si} = 0.1$. Note that temporal finetuning under CDGTD uses only 10\% of the data, so the total training cost remains 49+10×0.1=50 epochs. During temporal finetuning on each domain, we sample model weights at $K=5$ evenly spaced training steps and uniformly average them to obtain expert model weights. For PCA on expert deviations, we use the top 10 principal components. For ARIMA estimation, we employ an ARIMA(1,1,1) model. When computing averaging coefficients, we set the sharpness hyperparameter $r=0.5$.

\smallskip
\noindent\textbf{CLEAR-100}~\cite{Lin2022TheCB} adopts a ResNet-50~\cite{he2016deep}. All baseline methods are trained on 90\% randomly split training data from source domains for 50 epochs with batch size 128 and learning rate 5e-4 (except A-GEM which uses 1e-6). Other baseline configurations follow the FMoW configurations from \citet{xie2024evolving,xie2024weight}.

\smallskip
\noindent\textbf{TEA for CLEAR-100} uses the same ResNet-50~\cite{he2016deep}. Batch size is 128. Under TDG setting, TEA first trains on all source domain training splits using ERM for 45 epochs with learning rate 5e-4 during the pretraining stage, then performs temporal finetuning for 5 epochs on each domain in reverse temporal order (from D$_{5}$ to D$_1$) using SI with learning rate 1e-4 and constraint strength $c_{si} = 0.1$. Under CDGTD setting, we adopt 49-epoch incremental ERM training on each domain (from D$_1$ to D$_{5}$) with learning rate 5e-4, followed by 10 temporal finetuning epochs on each domain in reverse temporal order (from D$_{5}$ to D$_1$) using SI with learning rate 1e-4 and constraint strength $c_{si} = 0.1$. Note that temporal finetuning under CDGTD uses only 10\% of the data, so the total training cost remains 49+10×0.1=50 epochs. During temporal finetuning on each domain, we sample model weights at $K=5$ evenly spaced training steps and uniformly average them to obtain expert model weights. For PCA on expert deviations, we use the top 10 principal components. For ARIMA estimation, we employ an ARIMA(1,1,1) model. When computing averaging coefficients, we set the sharpness hyperparameter $r=0.5$.

\begin{figure*}[t]
    \centering
    \includegraphics[width=0.9\textwidth]{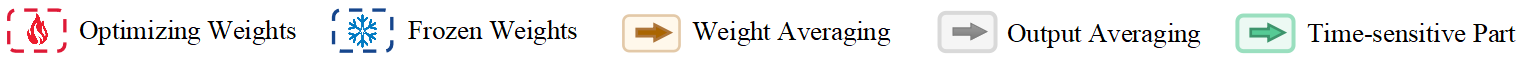}\\

    \begin{subfigure}{0.47\linewidth}
        \centering
        \includegraphics[width=0.99\linewidth]{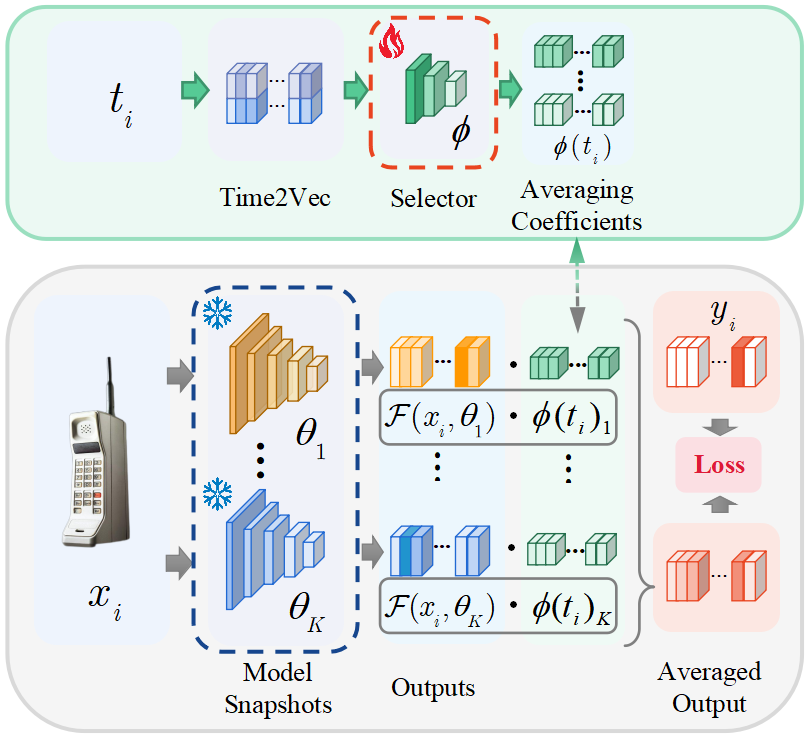}
        \caption{Training the Selector Network.}
        \label{fig:train_twa}
    \end{subfigure}
    \begin{subfigure}{0.47\linewidth}
        \centering
        \includegraphics[width=0.99\linewidth]{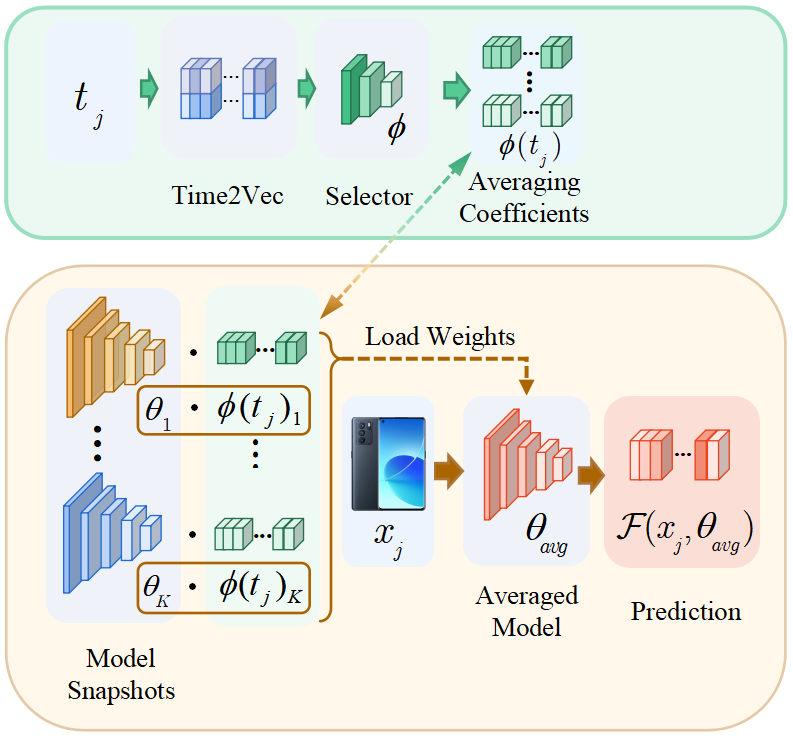}
        \caption{Inference.}
        \label{fig:infer_twa}
    \end{subfigure}
    
    \caption{An overview of our \textit{Only Adaptive Averaging} ablation. (a) When optimizing the selector network in \textit{Only Adaptive Averaging}, we use output averaging as a proxy task, utilizing the estimated coefficients to average the outputs of all snapshots. (b) During inference, we perform weight averaging with the optimized selector network. }
    \label{fig:twa_method}
    \vspace{-2.7mm}
\end{figure*}

\subsection{Ablation Details}
\label{sec:ablation_details}

Ablation study of TEA components examines four variants: Random Expert, Last Expert, Only Temporal Experts, and Only Adaptive Averaging. The first three involve simple modifications to specific TEA components, while Only Adaptive Averaging represents a more substantially different variant. We briefly describe the first three below and detail Only Adaptive Averaging in the following section:

\begin{itemize}
\item \textbf{Random Expert}: Randomly selects expert models and reports the average performance across multiple runs, which effectively equals the average performance of all experts.
\item \textbf{Last Expert}: Uses only the expert from the final domain.
\item \textbf{Only Temporal Experts}: Identical to TEA except for using uniform averaging coefficients ($1/S$) instead of adaptive coefficients to average all expert weights.
\end{itemize}

\noindent\textbf{Only Adaptive Averaging} shown in~\ref{fig:twa_method} aims to use base weights without temporal fine-tuning to achieve functional diversity, capturing temporal shift patterns solely through adaptive weight averaging in the coefficients. This variant cannot be implemented by simply removing TEA components, as our averaging coefficient estimation relies on shift patterns from experts corresponding to different temporal domains. Without temporal differences between base weights, we cannot use TEA's principal component trajectory-based coefficient estimation. Therefore, we adopt a training-based generation approach instead.

We first sample base weights. Following SWA~\cite{izmailov2018averaging}, we randomly sample S weights from the training process, which we call "snapshots". A key challenge arises with normalization layers: on TDG tasks, freezing normalization layers leads to underfitting, while optimizing them results in snapshots with different normalization parameters and statistics. Since weight averaging is highly sensitive to normalization differences, excessive variation causes poor performance in the averaged model. We address this using a "late sampling" strategy, as we observe that normalization becomes sufficiently good during intermediate training stages. Specifically, we freeze the normalization layers during the final epoch of each task and sample $K$ snapshots $\{\theta_k\}_{k=1}^K$ within this last epoch (noted as $K$ as we use all domain as a unified domain and set $K=S$ for fair ablation). We then generate adaptive averaging coefficients through a training-based approach. Specifically, we use a Time2Vec~\cite{kazemi2019time2vec} module with a 2-layer MLP as the selector network $\phi$ to generate averaging coefficients. After sampling the snapshots, we randomly select samples with timestamps from the training domains and train the selector network to combine the outputs of these snapshots. We formulate this training process as:
\adjustbox{width=0.49\textwidth}{%
$\displaystyle
\begin{aligned}
\label{eq:optim_meta}
\phi^{*} = & \mathop{\arg} \mathop{\min_{\phi}} \sum_{i \in [1, S]} \sum_{(x,t,y) \sim D_i}\ell \left( \sum_{k=1}^K \phi(t)_k \cdot f(x,  \theta_k), \: y \right)
\\
{\rm{s.t.}} & \quad \{\theta_k\}_{k=1}^K \sim  \mathcal{S}_{ls}(\mathop{\arg} \mathop{\min_{\theta}} \sum_{i \in [1, S]} \sum_{(X,\cdot,Y) \sim D_i} \ell(f(X,\theta),Y), 
\end{aligned}
$}
where $\mathcal{S}_{ls}$ is the snapshot sampling process with late sampling strategy. We use Adam optimizer for optimizing the selector network with learning rate as 1e-4, batch size as 1 and training steps as 2000.

After training $\phi^{*}$, we use it during inference to generate averaging coefficients for the $K=S$ snapshots: $\boldsymbol{\alpha}^{OAA} = \{\alpha^{OAA}_k\}_{k=1}^K = \phi^{*}(t_f)$. 
\section{Additional Discussion}

\textbf{TDG's Value for NLP Community.} On one hand, Temporal Domain Generalization (TDG)~\cite{OrtizJimnez2019CDOTCD, mancini2019adagraph, Wang2020ContinuouslyID, Bai2022TemporalDG, nasery2021training, Zeng2023ForeseeWY, wang2022evolving, xie2024evolving, xie2024enhancing, yong2023continuous, xie2024weight}  has broad application prospects in NLP tasks, as temporal distribution shifts are prevalent in NLP, such as lexical changes over time and evolving understanding of specific expressions (e.g., memes) across time periods. Particularly in the large language model era, TDG's low-resource generalization nature can reduce the expensive computational and data costs required for LLM retraining or fine-tuning. On the other hand, TDG has already been widely recognized as a valuable direction by the relevant community, with numerous papers published in top-tier conferences, including our baselines: GI (NeurIPS'21)~\cite{nasery2021training}, LSSAE (ICML'22)~\cite{qin2022generalizing}, DRAIN (ICLR'23)~\cite{Bai2022TemporalDG}, EvoS (NeurIPS'23)~\cite{xie2024evolving}, and W-Diff (NeurIPS'24)~\cite{xie2024weight}.

\smallskip
\noindent\textbf{Continual Learning.} TDG shares similar data configurations with continual learning~\cite{Zenke2017ContinualLT, LopezPaz2017GradientEM, Shin2017ContinualLW, Chaudhry2018EfficientLL}, and our main benchmarks ~\cite{yao2022wild,Lin2022TheCB}were originally introduced for continual learning. However, TDG and continual learning differ significantly in their objectives. Standard continual learning primarily focuses on the past, addressing whether learning new tasks causes catastrophic forgetting of previous knowledge. In contrast, TDG focuses on the future, concerned with leveraging past knowledge to enhance generalization to future domains. We incorporate representative continual learning baselines including EWC~\cite{kirkpatrick2017overcoming}, SI~\cite{Zenke2017ContinualLT}, and A-GEM~\cite{Chaudhry2018EfficientLL}, which show no significant generalization improvement on future domains.

\smallskip
\noindent\textbf{Continual Domain Generalization over Temporal Drift (CDGTD)} can be viewed as an intersection of standard TDG and continual learning. This represents a reasonable application direction, requiring models to both retain past knowledge and generalize well to future domains. However, this does not diminish the importance of standard TDG, as the core challenge of TDG—how to utilize temporal shift patterns in past data for better future generalization—is orthogonal to CDGTD's additional constraint of sequential domain access. Moreover, CDGTD may complicate the exploration of temporal generalization capabilities by introducing an additional variable. Therefore, we consider both standard TDG and CDGTD equally important, with no priority distinction.

\smallskip
\noindent\textbf{Large Language Models (LLMs).} While LLMs~\cite{OpenAI2023GPT4TR,Touvron2023LLaMAOA,guo2025deepseek} achieve good generalization through training on massive datasets, this does not conflict with TDG. TDG fundamentally targets low-resource scenarios and has considerable practical value when large training datasets are unavailable. Conversely, in cases of relatively smooth temporal distribution shifts, applying TDG with limited data is more data-efficient than brute-force generalization through massive training. Furthermore, regardless of how much data LLMs are trained on, TDG can be further applied to enhance temporal generalization capabilities. Notably, TDG application to LLMs is particularly promising as it can effectively reduce LLM training costs. However, TDG is still far from being applicable to LLMs, primarily due to scaling limitations. This highlights the value of our work as a solid step toward LLM-scale TDG.

\smallskip
\noindent\textbf{Temporal Reasoning}~\cite{xiong2024large,yuan2024back,fatemi2024test,chu2023timebench}. While this may sound related to TDG, the primary connection is that both contain "temporal" in their names. Temporal reasoning focuses on enabling models to understand explicit temporal relationships at the individual sample level, whereas TDG aims to adapt models to implicit temporal distribution shifts at the dataset level. Temporal reasoning could potentially improve TDG performance, but this remains unexplored.

\smallskip
\begin{table*}[t]
\centering
\resizebox{\textwidth}{!}{
\begin{tabular}{|c|c|c|c|c|c|c|c|}
\hline\hline
Method / Dataset & RMNIST & Yearbook & FMoW & HuffPost & Arxiv & CLEAR-10 & CLEAR-100 \\
\hline\hline
\#Samples   & 70,000   & 37,189   & 141,696   & 63,907   & 2,057,952   & 30,000   & 100,000 \\
\hline
Sample Size & 28$\times$28 & 32$\times$32 & 224$\times$224$\times$3 & - & - & 224$\times$224$\times$3 & 224$\times$224$\times$3 \\
\hline
\#Param     & 371.9K   & 29.0K   & 7.2M   & 66.4M   & 66.5M   & 11.2M   & 23.7M \\
\hline
GI~\cite{nasery2021training}          & \checkmark &   &   &   &   &   &   \\
LSSAE~\cite{qin2022generalizing}       & \checkmark & \checkmark &   &   &   &   &   \\
DRAIN~\cite{Bai2022TemporalDG}       & \checkmark & \checkmark &   &   &   &   &   \\
EvoS~\cite{xie2024evolving}        & \checkmark & \checkmark & \checkmark & \checkmark & \checkmark &   &   \\
W-Diff~\cite{xie2024weight}      & \checkmark & \checkmark & \checkmark & \checkmark & \checkmark &   &   \\
TEA (Ours)  & \checkmark & \checkmark & \checkmark & \checkmark & \checkmark & \checkmark & \checkmark \\
\hline\hline
\end{tabular}
}
\vspace{-2mm}
\caption{Dataset statistics and benchmark coverage across TDG methods. \checkmark\ indicates the method was originally explored and evaluated on the corresponding dataset in its original paper, demonstrating the evolution toward larger-scale TDG settings.}
\label{tab:dataset_methods}
\vspace{-2mm}
\end{table*}
\noindent\textbf{Scalability within TDG Domain.} Our scalability claims are contextualized within the TDG and CDGTD domains rather than general ML standards. Table~\ref{tab:dataset_methods} shows the evolution of benchmark coverage across TDG methods. Early methods like GI~\cite{nasery2021training}, LSSAE~\cite{qin2022generalizing}, and DRAIN~\cite{Bai2022TemporalDG} were limited to small benchmarks (RMNIST, Yearbook~\cite{yao2022wild}) with <40K samples and minimal parameters. Recent works like EvoS~\cite{xie2024evolving} and W-Diff~\cite{xie2024weight} expanded to medium-scale benchmarks (FMoW, HuffPost, Arxiv~\cite{yao2022wild}) with up to 2M samples and 66M parameters. TEA further extends this progression by introducing two new large benchmarks (CLEAR-10/100~\cite{Lin2022TheCB}) and demonstrating effectiveness across the full spectrum from 30K to 2M+ samples and 29K to 66M+ parameters. This represents significant scalability advancement within the TDG field, though the scale remains modest compared to contemporary NLP/CV standards.

\section{Theoretical Analysis}
\label{sec:theory}

\subsection{Notation}

We denote $\mathcal{X}$ the input space, $\mathcal{Y}$ the label space, and $\ell: \mathcal{Y}^2 \rightarrow \mathbb{R}^+$ a loss function. We have a sequence of domains $\{D_i\}$ indexed by timestamps $t_i \in \mathcal{T}$, where $\mathcal{T}$ is a totally ordered set representing time. Each domain $D_i$ has a distribution $p_i$. 
For the training (source) domains $\{D_i\}_{i=1}^{S}$, we have timestamps $t_1 < t_2 < \ldots < t_S$ in $\mathcal{T}$, and corresponding distributions $p_{1}, p_{2}, \ldots, p_{S}$. For simplicity, we will use $p_i$ to refer to the joint, posterior, and marginal distributions of $(X, Y)$ at time $t$.
We note $f_i: \mathcal{X} \rightarrow \mathcal{Y}$ as the labeling function at time $t_i$. We assume there is no noise in the data: $f_i$ is defined on $\mathcal{X}_i \triangleq \{x \in \mathcal{X} \mid p_i(x) > 0\}$ by $\forall(x, y) \sim p_i, f_i(x) = y$.

\subsection{Temporal Domain Generalization}

We consider a neural network (NN) $f(\cdot, \theta): \mathcal{X} \rightarrow \mathcal{Y}$ made of a fixed architecture $f$ with weights $\theta$. Given observations from source domains at times $t_1, t_2, \ldots, t_S$, we seek $\theta$ minimizing the target generalization error at a future time $t_f > t_S$:
\begin{equation}
\mathcal{E}_{f}(\theta) = \mathbb{E}_{(x,y)\sim p_{f}}[\ell(f(x, \theta), y)].
\end{equation}
$f(\cdot, \theta)$ should approximate $f_{f}$ on $\mathcal{X}_{f}$. This is challenging in the TDG setup because we only have data from earlier timestamps, which are related yet different from the future target domain.

The differences between domains at different timestamps are due to distribution shifts (i.e., the fact that $p_i(X, Y) \neq p_{j}(X, Y)$ for $i \neq j$), which can be decomposed into:
\begin{itemize}
    \item \textbf{Diversity shift:} when marginal distributions differ over time (i.e., $p_i(X) \neq p_{j}(X)$)
    \item \textbf{Correlation shift:} when posterior distributions differ over time (i.e., $p_i(Y|X) \neq p_{j}(Y|X)$ and $f_i \neq f_{j}$)
\end{itemize}

The weights are typically learned on source domain data $\{D_{1}, D_{2}, \ldots, D_{S}\}$ from timestamps $\{t_1, t_2, \ldots, t_S\}$ (each composed of $n_i$ i.i.d. samples from $p_i(X, Y)$) with a configuration $c$, which contains all other configurations and sources of randomness in learning. We call $l_\mathcal{T} = \{D_{1}, D_{2}, \ldots, D_{S}, c\}$ a learning procedure, and explicitly write $\theta(l_\mathcal{T})$ to refer to the weights obtained after stochastic minimization of the appropriate objective function. Specific to our TEA, we define $l_i = \{D_{1}, D_{2}, \ldots, D_{S}, t_i, c\}$ as a temporal expert learning procedure to get expert model $\theta_i=\theta(l_i)$ which is designed to excels on domain $D_{i}$ while also using data from other domains.

\subsection{Temporal Expert Averaging}

We study the benefits of combining $S$ individual member weights $\{\theta_i\}_{i=1}^S \triangleq \{\theta(l_{i})\}_{i=1}^S$ obtained from $S$ different domains at timestamps $\{t_1, t_2, \ldots, t_S\}$. Each weight $\theta_i$ corresponds to an expert model that is more proficient for domain $D_{i}$ (though not necessarily trained exclusively on that domain).

Unlike traditional weight averaging~\cite{cha2021swad,rame2022diwa,modelsoup} that uses equal coefficients, for temporal domain generalization, we propose a temporally-weighted averaging scheme that assigns different importance to experts based on their relevance to the target future domain. 

Temporal Expert Averaging (TEA) is defined as:
\begin{align}
    f_{\text{TEA}} &\triangleq f(\cdot, \theta_{\text{TEA}}), \nonumber\\
    \theta_{\text{TEA}} &\triangleq \sum_{i=1}^S \alpha_i\left(\{t_i\}_{i=1}^S, \{\theta_i\}_{i=1}^S,t_f\right) \cdot \theta_i.
\end{align}
where the coefficients $\{\alpha_i\}_{i=1}^S$ satisfy $\sum_{i=1}^S \alpha_i = 1$ and $\alpha_i \geq 0$ for all $i$. These coefficients are determined based on the temporal shift among the source domain experts $\{\theta_i\}_{i=1}^S$ and temporal information $\{t_i\}_{i=1}^S$ and $t_f$.

\subsection{TEA loss derivation}

Following \citet{rame2022diwa}, we decompose TEA’s error leveraging the similarity between WA and functional ensembling (ENS) \cite{lakshminarayanan2017simple,dietterich2000ensemble}, a more traditional way to combine a collection of weights. We also use Mean Squared Error as $\ell$ for simplicity. For TDG setting, we define Temporal ENS (T-ENS) with coefficients $\{\alpha_i\}_{i=1}^S$ as
\begin{align}
f_{\text{T-ENS}} \triangleq \sum_{i=1}^S \alpha_i f(\cdot, \theta_i). 
\end{align}

Lemma~\ref{lemma:talyor} establishes that $f_{\text{TEA}}$ approximates $f_{\text{T-ENS}}$ to first order when $\{\theta_i\}_{i=1}^S$ are close in weight space.

\begin{lemma}[TWA and T-ENS]
Given $\{\theta_i\}_{i=1}^S$ with learning procedures for different temporal experts. Denoting $\Delta_{\{\theta\}} = \max_{i=1}^S\|\theta_i - \theta_{\text{TEA}}\|_2$, $\forall(x, y) \in \mathcal{X} \times \mathcal{Y}$:
\begin{align}
 f_{\text{TEA}}(x) &= f_{\text{T-ENS}}(x) + O(\Delta^2_{\{\theta\}})  \\
 \ell(f_{\text{TEA}}(x), y) &= \ell(f_{\text{T-ENS}}(x), y) + O(\Delta^2_{\{\theta\}}).\nonumber
\end{align}
\label{lemma:talyor}
\end{lemma}

\begin{proof}
This proof has two components:
\begin{itemize}
    \item to establish the functional approximation, it performs Taylor expansion of the models' predictions at the first order.
    \item to establish the loss approximation, it performs Taylor expansion of the loss at the first order.
\end{itemize}
\textbf{Functional approximation} With a Taylor expansion at the first order of the models' predictions w.r.t. parameters $\theta$:

\begin{align}
&f_{\theta_i} = f_{\text{TEA}} + \nabla f |_{\text{TEA}}\Delta_i + O\left(\|\Delta_i\|^2_2\right) \nonumber\\
&f_{\text{T-ENS}} - f_{\text{TEA}} \nonumber\\
&= \sum_{i=1}^S \alpha_i \nabla f |_{\text{TEA}}\Delta_i
+ \sum_{i=1}^S \alpha_i O\left(\|\Delta_i\|^2_2\right) \nonumber,
\end{align}
where $\Delta_i = \theta_i - \theta_{\text{TWA}}$. 

Note that unlike in the equal weighting case, we don't have $\sum_{i=1}^S \Delta_i = 0$ for weighted averaging. Instead, we have $\sum_{i=1}^S \alpha_i \Delta_i = 0$. Therefore:
\begin{align}
&f_{\text{T-ENS}} - f_{\text{TEA}} \nonumber\\
&= \sum_{i=1}^S \alpha_i \nabla f |_{\text{TEA}}\Delta_i + \sum_{i=1}^S \alpha_i O\left(\|\Delta_i\|^2_2\right) \nonumber\\
&= \nabla f |_{\text{TWA}}\sum_{i=1}^S \alpha_i \Delta_i + O\left(\sum_{i=1}^S \alpha_i \|\Delta_i\|^2_2\right) \nonumber\\
&= O\left(\sum_{i=1}^S \alpha_i \|\Delta_i\|^2_2\right) \nonumber
\end{align}

Since $\Delta_i \leq \Delta_{\{\theta\}}$ for all $i$, and $\sum_{i=1}^S \alpha_i = 1$, we have:
\begin{align}
f_{\text{T-ENS}} - f_{\text{TEA}} &= O\left(\sum_{i=1}^S \alpha_i \Delta_{\{\theta\}}^2\right) \nonumber\\
&= O\left(\Delta_{\{\theta\}}^2 \sum_{i=1}^S \alpha_i\right) \nonumber\\
&= O\left(\Delta_{\{\theta\}}^2\right) \nonumber
\end{align}

\noindent\textbf{Loss approximation.} With a Taylor expansion at the zeroth order of the loss w.r.t. its first input and injecting the functional approximation:
\begin{align}
\ell(f_{\text{T-ENS}}(x); y) &= \ell(f_{\text{TWA}}(x); y) \nonumber\\ 
&+ O(\|f_{\text{T-ENS}}(x) - f_{\text{TEA}}(x)\|_2) \nonumber\\
\ell(f_{\text{T-ENS}}(x); y) &= \ell(f_{\text{TEA}}(x); y) + O\left(\Delta_{\{\theta\}}^2\right) \nonumber
\end{align}

\end{proof}


\subsection{Bias-variance-covariance-locality Decomposition for TEA}

We can derive the following decomposition of TEA's expected test error in the future domain. The expectation is over the joint distribution describing the $S$ learning procedures $\{l_{i}\}_{i=1}^S$. (\textit{Note that in the temporal domain generalization (TDG) setting, models from different timestamps may have different biases and variances due to the evolution of data distributions over time. This temporal heterogeneity is a key characteristic that distinguishes TDG from standard DG.})

\begin{proposition}[Bias-variance-covariance-locality decomposition for temporal weight averaging]
Denoting $\bar{f}_{i}(x) = \mathbb{E}_{l_{i}}[f(x, \theta(l_{i}))]$ as the expected prediction of an expert model for timestamp $t_i$, $\mathbb{E}_{f} = \mathbb{E}_{(x,y)\sim p_{f}}$ and $\mathbf{l} = \{l_1,\ldots,l_S\}$,  the expected generalization error on future domain $t_f$ of $\theta_{\text{TWA}} = \sum_{i=1}^S \alpha_i\cdot \theta_i$ over the joint distribution of the learning procedures is:
\begin{align}
&\mathbb{E}_{\mathbf{l}}[\mathcal{E}_{f}(\theta_{\text{TEA}})] = \mathbb{E}_{f}\left[ \mathcal{B} + \mathcal{V} +\mathcal{C}  \right] + O(\bar{\Delta}^2), \label{eq:bvcl_supp}
\end{align}
where 
\begin{align}
&\mathcal{B}=\left(\sum_{i=1}^S \alpha_i \cdot \text{bias}_i\right)^2, \ \text{bias}_i = y - \bar{f}_{i}(x), \nonumber\\
&\mathcal{V}=\sum_{i=1}^S \alpha_i^2 \cdot \text{var}_i,\  \text{var}_i = \mathbb{E}_{l_i}\left[ dev_i^2 \right],  \nonumber\\
&\mathcal{C} = \sum_{i \neq j} \alpha_i \alpha_j \text{cov}_{i,j}, \  \text{cov}_{i,j} = \mathbb{E}_{\{l_i,l_j\}}\left[ dev_i\cdot dev_j \right], \nonumber\\
& with\ dev_i = f(x, \theta(l_{i})) - \bar{f}_{i}(x),\nonumber\\
&\bar{\Delta}^2 = \mathbb{E}[\Delta^2_{\{\theta\}}] \text{ with } \Delta_{\{\theta\}} = \max_{i=1}^S \|\theta_i - \theta_{\text{TWA}}\|_2.\nonumber
\end{align}
\end{proposition}

\begin{proof}

Following \citet{rame2022diwa}, we use the he bias-variance decomposition in \citet{kohavi1996bias} with $f_{\text{T-ENS}} \triangleq \sum_{i=1}^S \alpha_i f(\cdot, \theta(l_{i}))$ to decompose the expected generalization error:
\begin{align}
&\mathbb{E}_{\mathbf{l}}[\mathcal{E}_{f}(\{\theta(l_{i})\}_{i=1}^S)] \nonumber\\
&= \mathbb{E}_{f}[\text{Bias}\{f_{\text{T-ENS}}|(x,y)\}^2 + \text{Var}\{f_{\text{T-ENS}}|x\}],\nonumber
\end{align}
where bias term becomes:
\begin{align}
&\text{Bias}\{f_{\text{T-ENS}}|(x,y)\} \nonumber\\
&= y - \mathbb{E}_{\mathbf{l}}\left[\sum_{i=1}^S \alpha_i f(x, \theta(l_{i}))\right] \nonumber\\
&= y - \sum_{i=1}^S \alpha_i \mathbb{E}_{\mathbf{l}}[f(x, \theta(l_{i}))] \nonumber\\
&= y - \sum_{i=1}^S \alpha_i \bar{f}_{i}(x) \nonumber\\
&= \sum_{i=1}^S \alpha_i (y - \bar{f}_{i}(x)) \nonumber\\
&= \sum_{i=1}^S \alpha_i \text{bias}_i(x, y) \nonumber
\end{align}

Thus, the squared bias term is:
\begin{align}
\text{Bias}\{f_{\text{T-ENS}}|(x,y)\}^2 \nonumber
= \left(\sum_{i=1}^S \alpha_i \text{bias}_i\right)^2 \nonumber
\end{align}

For the variance term, denoting $\text{dev}_i = f(x, \theta(l_{i})) - \bar{f}_{i}(x)$, we have:

\smallskip
\adjustbox{width=0.49\textwidth}{%
$\displaystyle
\begin{aligned}
&\text{Var}\{f_{\text{T-ENS}}|x\} \\
&= \mathbb{E}_{\mathbf{l}}\left[\left(\sum_{i=1}^S \alpha_i f(x, \theta(l_{i})) - \mathbb{E}_{\mathbf{l}}\left[\sum_{i=1}^S \alpha_i f(x, \theta(l_{i}))\right]\right)^2\right]\\
&= \mathbb{E}_{\mathbf{l}}\left[\left(\sum_{i=1}^S \alpha_i (f(x, \theta(l_{i})) - \bar{f}_{i}(x))\right)^2\right]\\
&= \mathbb{E}_{\mathbf{l}}\left[\sum_{i=1}^S \sum_{j=1}^S \alpha_i \alpha_j \cdot \text{dev}_i \cdot \text{dev}_j\right]\\
&= \sum_{i=1}^S \alpha_i^2 \mathbb{E}_{\mathbf{l}}[\text{dev}_i^2] + \sum_{i\neq j} \alpha_i \alpha_j \mathbb{E}_{\mathbf{l}}[\text{dev}_i \cdot \text{dev}_j]\\
&= \sum_{i=1}^S \alpha_i^2 \text{var}_i + \sum_{i\neq j} \alpha_i \alpha_j \text{cov}_{i,j}
\end{aligned}
$}

\noindent\textbf{Combination with Lemma 1} We recall that per our adapted Lemma 1:
\begin{align}
\ell(f_{\text{TEA}}(x), y) = \ell(f_{\text{T-ENS}}(x), y) + O(\Delta^2_{\{\theta\}}).\nonumber
\end{align}
Taking the expectation over the learning procedures and combining all terms:
\begin{align}
\mathbb{E}[\mathcal{E}_{f}(\theta_{\text{TEA}})] &= \mathbb{E}_{f}\left[ \left(\sum_{i=1}^S \alpha_i \text{bias}_i\right)^2\right] \nonumber\\
&+ \mathbb{E}_{f}\left[\sum_{i=1}^S \alpha_i^2 \text{var}_i\right] \nonumber\\
&+ \mathbb{E}_{f}\left[\sum_{i \neq j} \alpha_i \alpha_j \text{ cov}_{i,j} \right] \nonumber\\
&+ O(\bar{\Delta}^2)\nonumber
\end{align}
\end{proof}

\subsection{Theoretical Insights for TEA} 

From Equation \ref{eq:bvcl_supp}, we can see that generalization error can be reduced by minimizing bias $\mathcal{B}$, variance $\mathcal{V}$, covariance $\mathcal{C}$, and locality $\bar{\Delta}^2$. However, due to the complexity of real-world data and models, finding an optimal analytical solution is nearly impossible. Nevertheless, similar to \citet{rame2022diwa}, we can derive practical insights for designing TEA by analyzing the relationships between these four terms, model properties, and averaging coefficients.

\smallskip
\noindent\textbf{Insight 1} \textit{Tradeoff between Functional Diversity and Parameter Similarity among Experts.} Covariance $\mathcal{C}$ reduction necessitates functional diversity among experts, while the locality constraint $\bar{\Delta}^2$ demands parameter similarity among experts.

The covariance term increases when the predictions of $\{f(\cdot,\theta(l_{i}))\}_{i=1}^S$ are correlated, suggesting that DiWA's~\cite{rame2022diwa} approach to reduce covariance by encouraging functional diversity remains effective. However, the locality term $\bar{\Delta}^2$ simultaneously constrains the weights to remain close in parameter space. This tradeoff suggests that when training these expert models, we should find an appropriate balance between encouraging diverse predictions and maintaining parameter similarity.

\smallskip
\noindent\textbf{Insight 2} \textit{Tradeoff between Bias and Variance via Averaging Coefficients.} Reducing variance $\mathcal{V}$ requires averaging weights evenly, while reducing bias $\mathcal{B}$ demands concentrating coefficients on experts with lower bias magnitudes on future data. 

Insight 2 is obtained by introducing 2 assumptions specific to the TDG for further discussion about bias and variance. 

\begin{assumption}[Ordered Bias Magnitudes]
The models can be ordered by expected bias magnitudes on future domains such that $\mathbb{E}_{f}\left[\text{bias}_{m_1}^2\right] \geq \mathbb{E}_{f}\left[\text{bias}_{m_1}^2\right] \geq \cdots \geq \mathbb{E}_{f}\left[\text{bias}_{m_S}^2\right]$, with $\{m_j\}_{j=1}^S$ being a permutation of $\{i\}_{i=1}^S$.
\label{insight:diverse_similar}
\label{assump:bias}
\end{assumption}

\begin{assumption}[Equal Variance Experts]
The variance of each expert's prediction is equal across all experts, such that $\mathbb{E}_{f}\left[\text{var}_i\right] = v$ for all $i \in \{1,2,...,M\}$.
\label{assump:var}
\end{assumption}

\begin{lemma}[\textbf{Optimal Averaging Coefficients for Bias Minimization}]
\label{lemma:optimal_bias_coeff}
Let the bias of model $i$ be:
\begin{align}
b_i(x,y) = \operatorname{bias}_i(x,y),
\sigma_i^{2} := \mathbb{E}_{f}\left[b_i^{2}\right],\nonumber
\end{align}
and define the root-mean-square magnitude:
\begin{align}
\sigma_i = \sqrt{\sigma_i^{2}} \quad (i=1,\ldots,S).\nonumber
\end{align}
According to Assumption~\ref{assump:bias}, magnitudes are ordered $\sigma_{m_1}\geq\sigma_{m_2}\geq\cdots\geq\sigma_{m_S}$, where $\{m_j\}_{j=1}^{S}$ is a permutation of $\{1,\ldots,S\}$. For convex weights $\boldsymbol{\alpha}\in\Delta^{S}:=\{\alpha_i\geq0,\, \sum_{i=1}^{S}\alpha_i=1\}$, consider the combined bias loss:
\begin{align}
L(\boldsymbol{\alpha}) := \mathbb{E}_{f}\left[\left(\sum_{i=1}^{S}\alpha_i b_i\right)^{2}\right].
\end{align}
If \textbf{no information} is available on the pairwise bias covariances $\Sigma_{ij}:=\mathbb{E}_{f}[b_i b_j],$ $(i\neq j)$, then the minimax problem:
\begin{align}
\min_{\boldsymbol{\alpha}\in\Delta^{S}}\; \max_{\Sigma\text{ s.t. }\operatorname{diag}(\Sigma)=\boldsymbol{\sigma}^{2}} L(\boldsymbol{\alpha})
\end{align}
is solved by:
\begin{align}
\boxed{\alpha^{\star}_{m_S}=1,\quad \alpha^{\star}_{i}=0 \;\text{for}\;i\neq m_S}
\end{align}
with $L(\boldsymbol{\alpha}^{\star})=\sigma_{m_S}^{2}$.
\end{lemma}

\begin{proof}
We can write $L(\boldsymbol{\alpha})=\boldsymbol{\alpha}^{\top}\Sigma\boldsymbol{\alpha}$ with unknown positive-semidefinite matrix $\Sigma$ satisfying $\Sigma_{ii}=\sigma_i^{2}$. By the Cauchy-Schwarz inequality, $|\Sigma_{ij}|\leq\sigma_i\sigma_j$. The worst case occurs when all covariances reach the extreme value $\Sigma_{ij}=\sigma_i\sigma_j$, yielding:
\begin{align}
\max_{\Sigma}L(\boldsymbol{\alpha}) = \left(\sum_{i=1}^{S}\alpha_i\sigma_i\right)^{2}.
\end{align}
Since $\sum_{i}\alpha_i\sigma_i$ is a convex combination of the ordered set $\{\sigma_{m_j}\}$, its minimum over the simplex $\Delta^{S}$ is attained by placing all weight on the smallest RMS magnitude $\sigma_{m_S}$, which gives the stated $\boldsymbol{\alpha}^{\star}$ and the minimax value $L(\boldsymbol{\alpha}^{\star})=\sigma_{m_S}^{2}$.
\end{proof}

\begin{lemma}[Optimal Averaging Coefficients for Variance Minimization]
\label{lemma:optimal_variance_coeff}
Consider the variance term with equal variances $\mathbb{E}_{f}\left[\text{var}_i\right] = v$ for all $i \in \{1,\ldots,S\}$:
\begin{align}
\mathbb{E}_{f}\left[\mathcal{V}\right] = v\sum_{i=1}^S \alpha_i^2.
\end{align}
For averaging coefficients $\boldsymbol{\alpha} \in \Delta^S := \{\alpha_i \geq 0, \sum_{i=1}^S \alpha_i = 1\}$, the variance term is minimized when weights are distributed equally across all models:
\begin{align}
\boxed{\alpha_i^{\star} = \frac{1}{S} \text{ for all } i}
\end{align}
with optimal variance $v \cdot \frac{1}{S}$.
\end{lemma}

\begin{proof}
We seek to minimize $\sum_{i=1}^S \alpha_i^2$ subject to the constraints $\sum_{i=1}^S \alpha_i = 1$ and $\alpha_i \geq 0$. By the Cauchy-Schwarz inequality:
\begin{align}
\left(\sum_{i=1}^S \alpha_i\right)^2 \leq S \sum_{i=1}^S \alpha_i^2,
\end{align}
with equality if and only if all $\alpha_i$ are equal. Since $\sum_{i=1}^S \alpha_i = 1$, we have:
\begin{align}
1 = \left(\sum_{i=1}^S \alpha_i\right)^2 \leq S \sum_{i=1}^S \alpha_i^2,
\end{align}
which implies $\sum_{i=1}^S \alpha_i^2 \geq \frac{1}{S}$. Equality is achieved when $\alpha_i = \frac{1}{S}$ for all $i$, giving the optimal solution.
\end{proof}

In summary, Lemma~\ref{lemma:optimal_bias_coeff} indicates that optimizing the bias term requires concentrating weight on experts with smaller bias magnitude on future domains, while Lemma~\ref{lemma:optimal_variance_coeff} suggests that minimizing variance requires the opposite approach—distributing weight as evenly as possible across all experts. This creates a fundamental tradeoff between bias and variance in the selection of averaging coefficients.

 \smallskip
\noindent\textbf{Discussion about Assumptions.} Assumption~\ref{assump:bias} is similar to the smooth distribution shift assumption used by most prior TDG methods~\cite{Bai2022TemporalDG,Zeng2023ForeseeWY,nasery2021training,xie2024weight,xie2024evolving}, allowing us to model distribution change and leverage temporal information to predict future parameter or feature. Assumption~\ref{assump:var} is reasonable when all experts share the same architecture, optimization procedure and hyperparameters, differing only in the specific temporal domains they've been optimized to excel in.
\section{Additional Results}

\begin{figure*}[t]
    \centering
    \begin{subfigure}[t]{0.45\textwidth}
        \includegraphics[width=\linewidth]{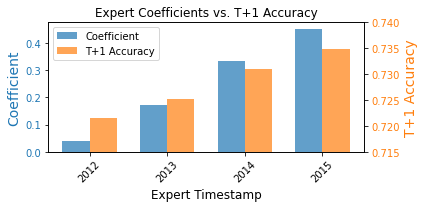}
        \caption{Huffpost.}
        \label{fig:huffpost_coeffs_supp}
    \end{subfigure}
    \begin{subfigure}[t]{0.45\textwidth}
        \includegraphics[width=\linewidth]{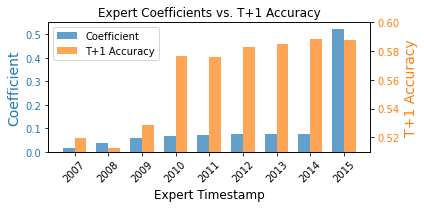}
        \caption{Arxiv.}
        \label{fig:arxiv_coeffs_supp}
    \end{subfigure}
    \par\bigskip
    \begin{subfigure}[t]{0.45\textwidth}
        \includegraphics[width=\linewidth]{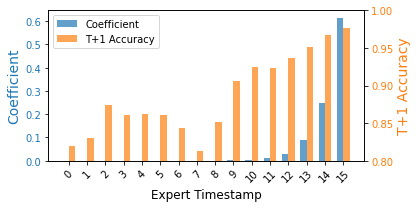}
        \caption{Yearbook.}
        \label{fig:yearbook_coeffs_supp}
    \end{subfigure}
    \begin{subfigure}[t]{0.47\textwidth}
        \includegraphics[width=\linewidth]{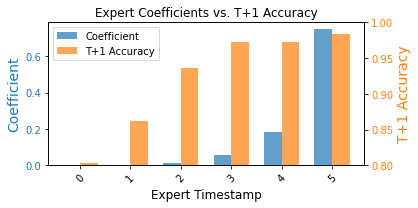}
        \caption{RMNIST.}
        \label{fig:rmnist_coeffs_supp}
    \end{subfigure}

    \par\bigskip 

    \begin{subfigure}[t]{0.45\textwidth}
        \includegraphics[width=\linewidth]{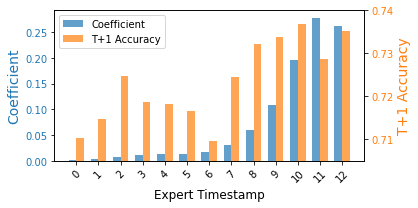}
        \caption{FMoW.}
        \label{fig:fmow_coeffs_supp}
    \end{subfigure}
    \begin{subfigure}[t]{0.46\textwidth}
        \includegraphics[width=\linewidth]{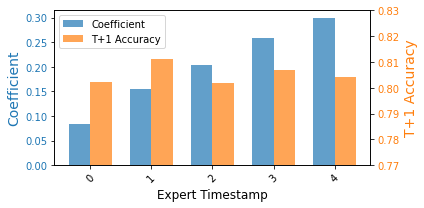}
        \caption{CLEAR-10.}
        \label{fig:clear10_coeffs_supp}
    \end{subfigure}

    \par\bigskip

    \begin{subfigure}[t]{0.45\textwidth}
        \includegraphics[width=\linewidth]{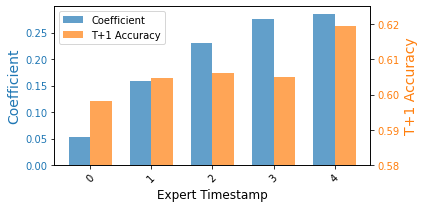}
        \caption{CLEAR-100.}
        \label{fig:clear100_coeffs_supp}
    \end{subfigure}

    \caption{Visualization of averaging coefficients and accuracies of experts on target domain $D_{S+1}$.}
    \label{fig:coeff_ablation_supp}
\end{figure*}

We show the coefficients vs. $D_{S+1}$ accuracy across all benchmarks in Figure~\ref{fig:coeff_ablation_supp}. The results reveal diverse performance patterns among experts, with a clear temporal trend: experts trained on earlier domains consistently achieve lower accuracy on the subsequent domain. This systematic pattern demonstrates that expert functionality undergoes structured evolution throughout the finetuning process. In addition, our temporal trajectory forecasting approach successfully identifies this pattern and assigns higher coefficients to experts that demonstrate superior performance on future domains, effectively leveraging the temporal dynamics to improve expert selection and utilization.

\end{document}